%% file: main.tex
\PassOptionsToPackage{dvipsnames, table, xcdraw}{xcolor}
\documentclass[]{usc}

\input{macros.tex}
\input{sections/preamble}

\title{Solve the Loop: Attractor Models for Language and Reasoning}

\author{Jacob Fein-Ashley}
\author{Paria Rashidinejad}
\vspace{-2em}
\affiliation{University of Southern California}

\date{May 12, 2026}
\website{https://attractor-models.github.io/}
\code{https://github.com/jacobfa/Attractor} 
\correspondence{\email{\{feinashl,paria.rashidinejad\}@usc.edu}}

\abstract{\input{sections/abstract}}

\begin{document}

\maketitle

\input{sections/intro.tex}

\input{sections/method.tex}
\input{sections/experiments}
\input{sections/conclusion.tex}

\bibliographystyle{assets/plainnat}
\bibliography{main}

\appendix

\input{sections/supplementary.tex}

\end{document}

%% file: macros.tex
\usepackage{ifthen}  %
\newboolean{set_me_to_remove_todos}
\setboolean{set_me_to_remove_todos}{false} %

\ifthenelse{\boolean{set_me_to_remove_todos}}{
    \newcommand{\pierre}[1]{}
    \newcommand{\hady}[1]{}
    \newcommand{\tomas}[1]{}
    \newcommand{\alex}[1]{}
    \newcommand{\todo}[1]{}
    \newcommand{\remove}[1]{}
}{
    \usepackage[normalem]{ulem}
    \newcommand{\remove}[1]{{\color{red} \sout{#1}}}
    \newcommand{\pierre}[1]{{\color{blue} [\textbf{Pierre}: #1]}}
    \newcommand{\hady}[1]{{\color{purple} [\textbf{Hady}: #1]}}
    \newcommand{\tomas}[1]{{\color{orange} [\textbf{Tomas}: #1]}}
    \newcommand{\alex}[1]{{\color{olive} [\textbf{Alex}: #1]}}
    
    \newcommand{\todo}[1]{{\color{red} [\textbf{TODO}: #1]}}
}

\newcommand{\cmark}{\ding{51}}%
\newcommand{\xmark}{\ding{55}}%

\def\1{\mathbbm{1}}

\newlength\savewidth

\usepackage{float} %

\usepackage[most]{tcolorbox}

\definecolor{TakeawayBg}{HTML}{F7F7F7}
\definecolor{TakeawayRule}{HTML}{777777}
\definecolor{TakeawayTitle}{HTML}{333333}

\newtcolorbox{takeawaysbox}{
    enhanced,
    breakable,
    colback=TakeawayBg,
    colframe=TakeawayBg,
    borderline west={1.5pt}{0pt}{TakeawayRule},
    boxrule=0pt,
    arc=0pt,
    left=8pt,
    right=7pt,
    top=7pt,
    bottom=7pt,
    before skip=8pt,
    after skip=8pt
}

\newcommand{\gainup}[1]{\textcolor{ForestGreen}{\scriptsize $\uparrow$#1\%}}
\newcommand{\gaindown}[1]{\textcolor{ForestGreen}{\scriptsize $\downarrow$#1\%}}

\usepackage{lmodern}

%% file: sections/preamble.tex
\usepackage{booktabs}
\usepackage{graphicx}
\usepackage{url}
\usepackage{amsmath, amsthm, amsfonts, amssymb, bbm}
\usepackage{dsfont}
\usepackage{multirow}
\usepackage{array}
\usepackage{pifont}
\usepackage{stfloats}
\usepackage{makecell}
\usepackage{siunitx}
\usepackage{mathtools}
\usepackage{algorithm}
\usepackage{mathrsfs}
\usepackage{tabularray}
\usepackage{sansmath}
\usepackage{algpseudocode}
\usepackage[most]{tcolorbox}
\usepackage{tabularx}
\usepackage{amsthm}
\usepackage{float}
\usepackage{placeins}
\usepackage{graphicx}
\usepackage{subcaption}
\usepackage{amsmath, amsfonts,amssymb}
\usepackage{bm}
\usepackage{caption}
\usepackage{amsthm}
\usepackage{needspace}
\tcbuselibrary{skins}
\UseTblrLibrary{booktabs}

\makeatletter
\NAT@numberstrue
\renewcommand\NAT@open{[}
\renewcommand\NAT@close{]}
\makeatother
\setcitestyle{authoryear,round}
\let\cite\citep

\definecolor{TableHeader}{HTML}{EFF5FB}
\definecolor{TableBand}{HTML}{F7FAFD}
\definecolor{TableRule}{HTML}{9DAFBE}
\definecolor{TableAccent}{HTML}{3F6380}
\definecolor{TableText}{HTML}{202830}
\definecolor{Gray}{gray}{0.95}

\newcolumntype{Y}{>{\centering\arraybackslash}X}
\newcommand{\grouprule}{\arrayrulecolor{TableRule}\specialrule{0.52pt}{0pt}{0pt}}
\newcommand{\HeaderStrut}{\rule{0pt}{3.15ex}}

\newtheorem{assumption}{Assumption}
\newtheorem{theorem}{Theorem}
\newtheorem{corollary}{Corollary}

\definecolor{BestHighlightBase}{HTML}{FFF2A8}
\colorlet{BestHighlight}{BestHighlightBase!35!white}
\newcommand{\best}[1]{\begingroup\setlength{\fboxsep}{1.2pt}\colorbox{BestHighlight}{#1}\endgroup}

\usepackage{enumitem}
\usepackage{wrapfig}

\definecolor{rliableolive}{HTML}{BBCC33}
\definecolor{rliableblue}{HTML}{77AADD}
\definecolor{rliablered}{HTML}{EE8866}

\definecolor{takeawaycolor}{HTML}{FFF1E6}
\definecolor{takeawaycolor2}{HTML}{F2F6FF}
\definecolor{takeawaycolor3}{HTML}{EAF6EF}
\definecolor{takeawaycolor4}{HTML}{E6F4F1}
\definecolor{takeawaycolor5}{HTML}{FFF9E5}
\definecolor{takeawayborder}{HTML}{FFD8C2}
\definecolor{takeawayheader}{HTML}{0B3C8C}

\tcbset{
  aibox/.style={
    width=\linewidth,
    top=9pt,
    bottom=2pt,
    left=4pt,
    right=4pt,
    colback=takeawaycolor5!100!white,
    colframe=black!75!white,
    colbacktitle=black!75!white,
    boxrule=0.75pt,
    enhanced,
    center,
    attach boxed title to top left={yshift=-0.1in,xshift=0.15in},
    boxed title style={boxrule=0pt,colframe=white,},
  }
}
\newtcolorbox{AIbox}[2][]{aibox,title=#2,#1}

\tcbset{
  greenaibox/.style={
    width=\linewidth,
    top=9pt,
    bottom=4pt,
    colback=takeawaycolor3!100!white,
    colframe=black!75!white,
    colbacktitle=black!75!white,
    boxrule=0.75pt,
    enhanced,
    center,
    attach boxed title to top left={yshift=-0.1in,xshift=0.15in},
    boxed title style={boxrule=0pt,colframe=white,},
  }
}
\newtcolorbox{greenAIbox}[2][]{greenaibox,title=#2,#1}

\newcommand{\cT}{\mathcal{T}}
\newcommand{\cC}{\mathcal{C}}
\newcommand{\cN}{\mathcal{N}}

\newcommand{\cR}{\mathcal{R}}
\newcommand{\cP}{\mathcal{P}}

\newcommand{\cA}{\mathcal{A}}

%% file: sections/abstract.tex
\vspace{-0.1em}
Looped Transformers offer a promising alternative to purely feed-forward computation by iteratively refining latent representations, improving language modeling and reasoning. Yet recurrent architectures remain unstable to train, costly to optimize and deploy, and constrained to small, fixed recurrence depths. We introduce \emph{Attractor Models}, in which a backbone module first proposes output embeddings, then an attractor module refines them by solving for the fixed point, with gradients obtained through implicit differentiation. Thus, training memory remains constant in effective depth, and iterations are chosen adaptively by convergence. Empirically, Attractor Models outperform existing models across two regimes, large-scale language-model pretraining and reasoning with tiny models. In language modeling, Attractor Models deliver a \emph{Pareto improvement} over standard Transformers and stable looped models across sizes, improving perplexity by up to 46.6\% and downstream accuracy by up to 19.7\% while reducing training cost. Notably, a 770M Attractor Model outperforms a 1.3B Transformer trained on twice as many tokens. On challenging reasoning tasks, we show that our model with only 27M parameters and approximately 1000 examples achieves 91.4\% accuracy on Sudoku-Extreme and 93.1\% on Maze-Hard, scaling favorably where frontier models like Claude and GPT o3, fail completely, and specialized recursive reasoners collapse at larger sizes. Lastly, we show that Attractor Models exhibit a novel phenomenon, which we call \emph{equilibrium internalization}: fixed-point training places the model's initial output embedding near equilibrium, allowing the solver to be removed at inference time with little degradation. Together, these results suggest that Attractor Models make iterative refinement scalable by turning recurrence into a computation the model can learn to internalize.\vspace{-0.2em}

%% file: sections/intro.tex
\begin{figure}[H]
    \centering

    \begin{subfigure}[t]{0.5\textwidth}
        \vspace{20pt}
        \centering
        \includegraphics[width=\linewidth]{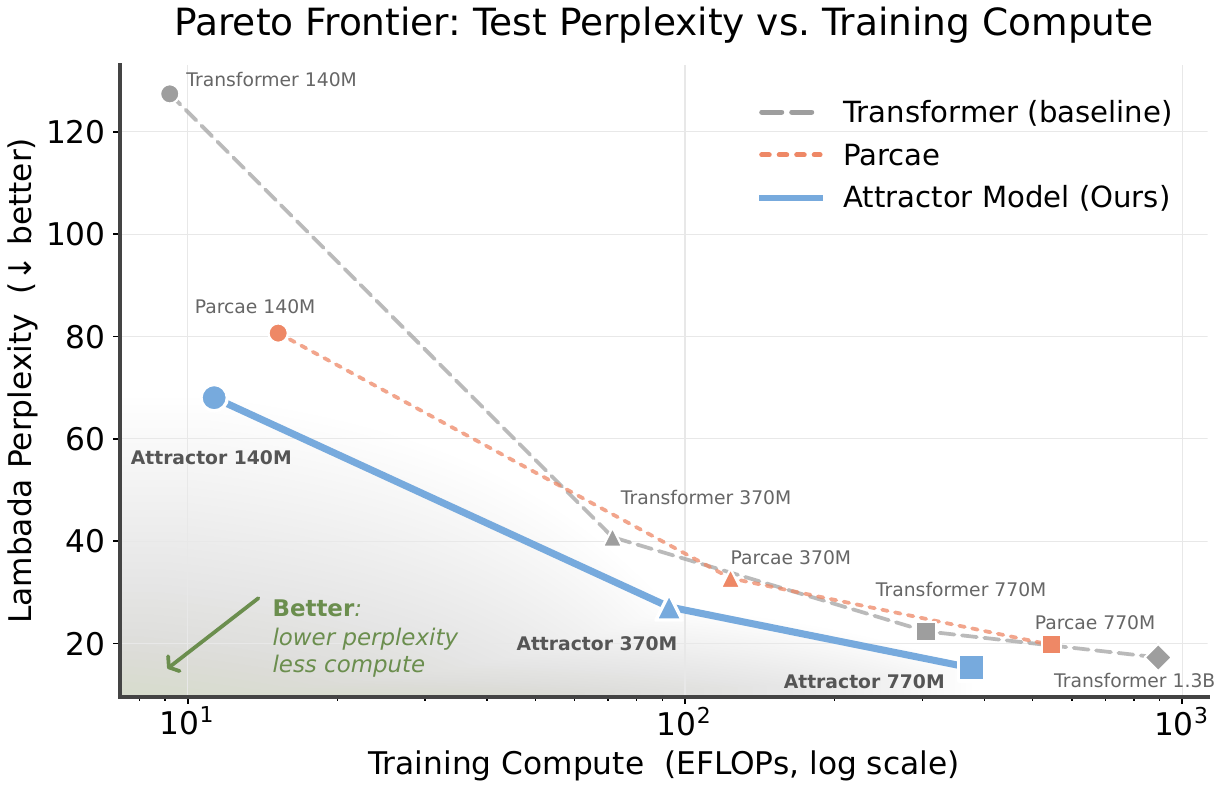}
        \label{fig:pareto_ppl_vs_flops}
    \end{subfigure}
    \hfill
    \begin{minipage}[t]{0.48\textwidth}
        \vspace{0pt}
        \centering

        \begin{subfigure}[t]{\linewidth}
            \centering
            \includegraphics[width=\linewidth]{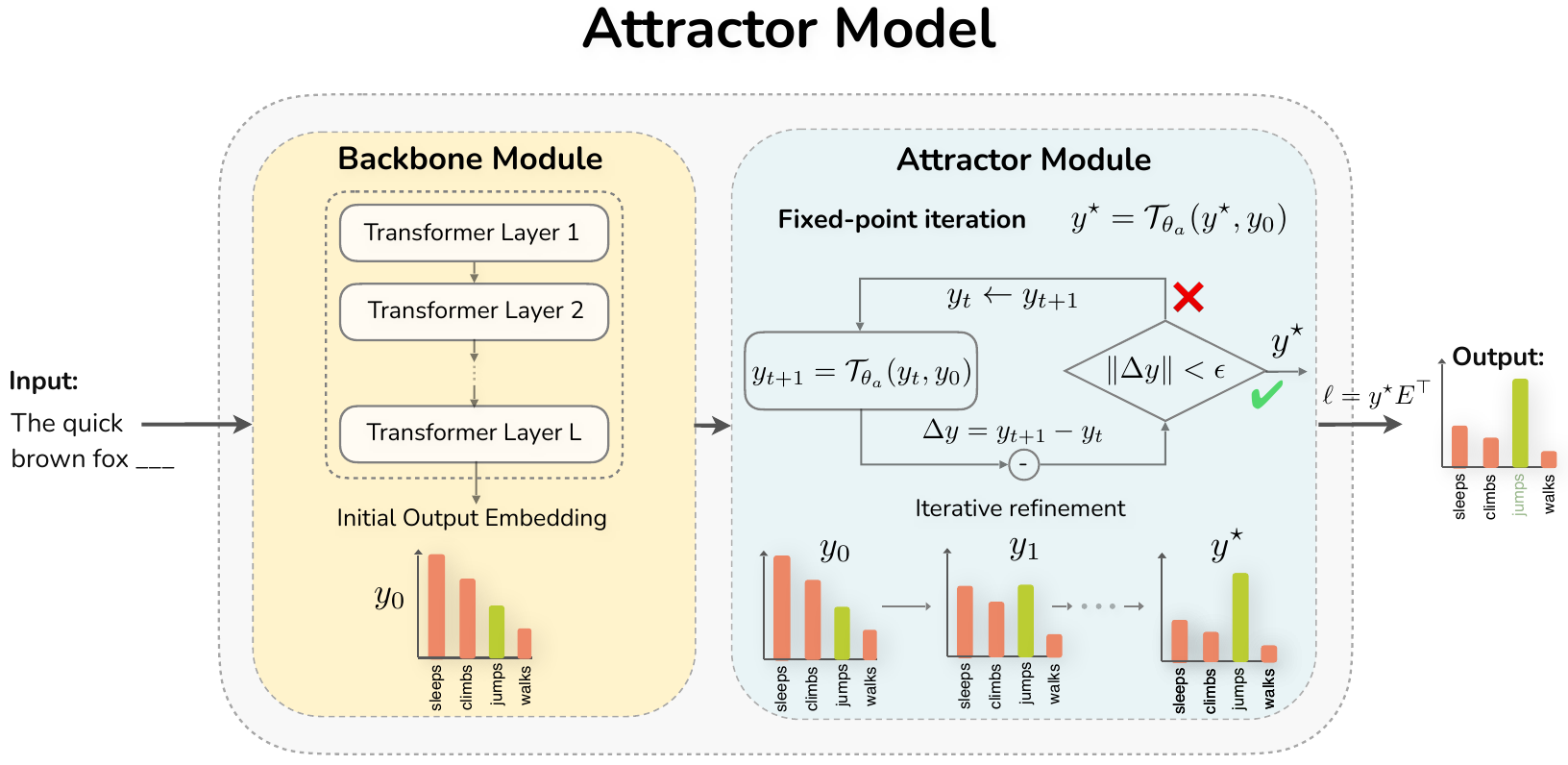}
            \label{fig:attractor_model_arch}
        \end{subfigure}
        \vspace{-0.8em}
        \begin{subfigure}[t]{\linewidth}
            \centering
            \includegraphics[width=\linewidth]{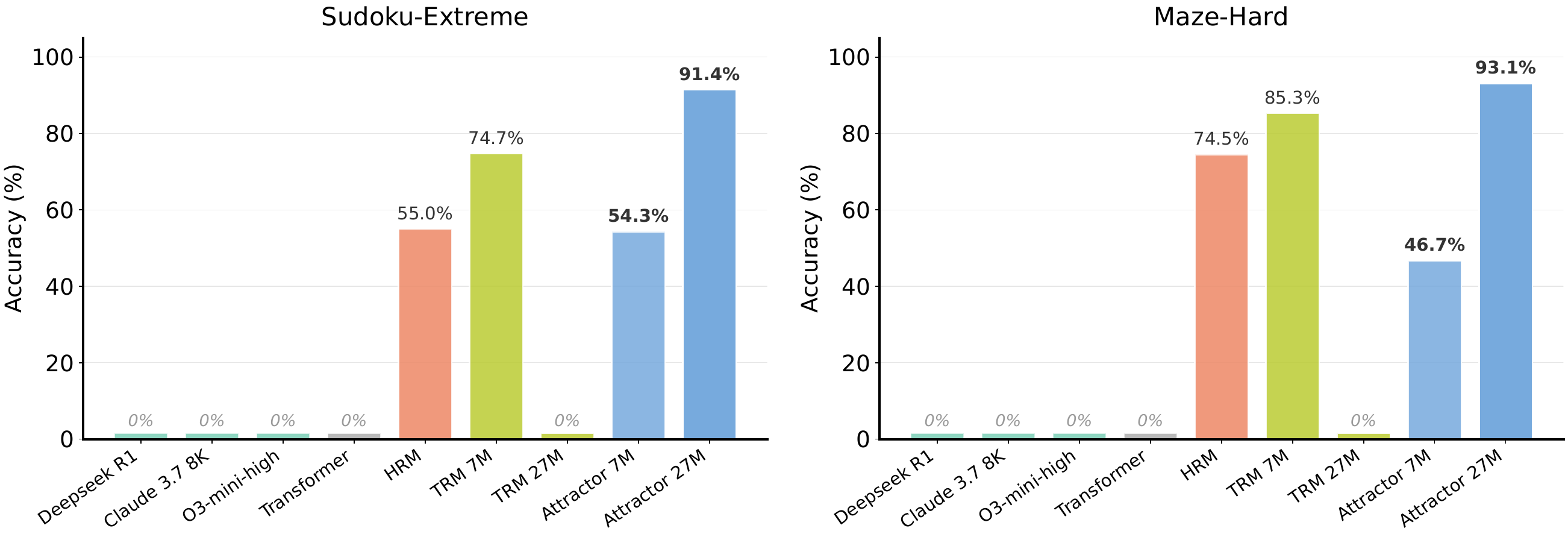}
            \label{fig:sudoku_maze}
        \end{subfigure}
    \end{minipage}
    \vspace{-1em}
    \caption{\textbf{Left: Lambada perplexity versus training FLOPs.} Attractor Models achieve strong language-modeling performance with less compute. \textbf{Top Right: Attractor Model architecture.} A non-recurrent backbone first maps the input to an initial output embedding proposal, which the attractor module refines by solving a fixed-point iteration before decoding the (approximate) equilibrium into the final output distribution. \textbf{Bottom Right: Performance on hard reasoning tasks.}  Attractor Models outperform frontier and specialized recursive models on hard reasoning tasks.}
    \label{fig:pareto_arch_sudoku_maze}
\end{figure}

\section{Introduction}
\label{sec:intro}

The modern language-modeling era has been dominated by Transformers \cite{transformer}, which produce each token through a fixed feed-forward computation. This recipe has been extraordinarily successful \cite{{achiam2023gpt, team2023gemini, grattafiori2024llama, anthropic2024claude3, r1}}, but it leaves a basic question unresolved: should each token be the product of a single pass of computation, or should a model be able to refine its latent prediction before committing to an output? A growing body of work suggests that such refinement can be powerful. Chain-of-thought reasoning \cite{cot} can be viewed as one form of such refinement, where a model writes intermediate tokens, feeds them back into its context, and uses them to shape later predictions. Yet this routes computation through the discrete token channel and forces  ``thinking'' to be written down, even when the effect might be to merely refine internal representations.

This limitation has inspired several lines of work on latent (or implicit) thinking and a re-emergence of architectural recurrence, which move thinking away from purely token-level generation. These include universal Transformers \cite{universal_transformer}, looped Transformers \cite{giannou2023looped, loop2}, recurrent-depth Transformers \cite{loop4}, looped language models \cite{looplm}, latent reasoning methods \cite{geiping, loop5}, and continuous chain-of-thought approaches \cite{coconut, mohtashami2023cotformer, zhu2026reasoning}. Looped architectures can, in principle, express iterative or algorithmic procedures \cite{loopedbetter, giannou2023looped}, emulate additional depth through weight sharing \cite{universal_transformer, looplm}, reduce the context-length costs of token-level reasoning, and improve downstream generalization \cite{loop3, loopedtransformer}. Empirically, recent looped language models offer gains in language modeling and reasoning \cite{looplm, geiping}, and tiny recursive models \cite{hrm, trm} have shown that recurrence can be useful in hard reasoning tasks in small-data regimes.

The challenge is that recurrence has proven difficult to use as a stable architectural building block. Recurrence is often accompanied by unstable training, large memory requirements that grow linearly with the number of recurrent steps, and significant, sequential compute~\cite{geiping, looplm, parcae}. Training recurrent networks typically requires backpropagation through time (or, depth) and carefully designed stabilization techniques; even then, latent-thinking models remain fragile and difficult to optimize \cite{wei2025sim, ozeren2025reinforcement, deng2026latent,deng2025latent, rizvi2026illusion}. For training, looped language models tend to require substantially more compute than comparable feed-forward models and can become memory-limited at larger recurrence depths. For instance, \citet{geiping} reports that training a recurrent model can consume raw FLOPs comparable to those of a feed-forward model ten times larger. At the opposite end of the spectrum, specialized tiny recursive reasoners exhibit a troubling ``less is more'' behavior and respond negatively to scaling: increasing model size can degrade or even collapse performance \cite{trm}.

\subsection{Contributions}
In this work, we design a general-purpose architecture for iterative refinement that is (i) stable to train, (ii) uses constant-memory in the number of refinement steps, (iii) is substantially cheaper to train than explicit unrolling, (iv) is efficient during inference, and (v) achieves a strong performance across both large-scale language modeling and hard reasoning with tiny models.

\textbf{Refine outputs by solving the loop with Attractor Models.} We introduce \textit{Attractor Models}, a new family of architectures that treat latent refinement as a fixed-point problem in the output embedding space. The model first proposes an initial guess embedding using a non-recurrent backbone module (implemented as a Transformer in ours). A separate, typically smaller, recurrent network then refines this guess (Figure~\ref{fig:pareto_arch_sudoku_maze}). Recent mechanistic analyses of looped language models demonstrate that, for the vast majority of tokens, the recurrent trajectory converges to a fixed point \cite{mech}. We build directly on this observation and instead of unrolling the loop for a predefined number of steps, we solve for the state to which the loop converges, inspired by Deep Equilibrium Models (DEQ; \cite{deq}). The name {Attractor Model} comes from dynamical systems, where an attractor is a set of states toward which a system evolves. In a sense, Attractor Models can be viewed as \emph{thinking} before producing each token: the backbone proposes an initial latent prediction, the attractor module refines it to equilibrium, then decodes it into the output distribution.

\textbf{Attractor Models offer stable, constant-memory, efficient training, and adaptive refinement.} Unlike looped LMs, which finitely unroll the recurrent block, Attractor Models solve an equilibrium by treating the prediction target as a fixed-point computation. The number of refinement steps is therefore chosen adaptively according to convergence during both training and inference. We show that the memory cost during training remains constant in the number of iterations; whereas standard looped language models have a linear scaling increase with the number of loops. Our experiments demonstrate that the two-stage structure of Attractor Models, in which the backbone proposes and the attractor refines, enables stable, efficient training and strong performance.

\textbf{Novel phenomenon: Equilibrium internalization.} We observe that despite the fact that Attractor models are trained only with the next-token prediction loss, they learn to make the solver unnecessary. During training, the backbone's initial prediction moves progressively closer to the fixed point, so fewer refinement steps are needed to reach approximate equilibrium (c.f. Figures \ref{fig:convergence_behavior} and \ref{fig:accuracy_vs_T}). We call this phenomenon \emph{equilibrium internalization}: the model appears to self-distill the iterative refinement process into its own initial output embedding, through a form of automatic curriculum. In this sense, recurrence acts as a moving training target, teaching the backbone where its computation should converge.

\textbf{Strong performance in large-scale language modeling and hard reasoning with tiny models.} Our experiments show that Attractor Models scale across two regimes. In large-scale language modeling, Attractor Models consistently outperform standard Transformers and stable looped language models across small (140M), medium (370M), and large (770M) sizes, delivering a Pareto improvement (Figure~\ref{fig:pareto_arch_sudoku_maze}). We show that our models improve validation perplexity, out-of-distribution perplexity on Lambada~\cite{lambada}, and downstream benchmark accuracy while using substantially less training compute than comparable looped baselines. Notably, a 770M-parameter Attractor Model outperforms a 1.3B-parameter Transformer trained on twice as many tokens. Compared to looped LM Parcae \cite{parcae}, our models use up to $31\%$ less training compute, while avoiding the memory growth associated with explicit unrolling. In hard reasoning tasks with tiny models, with only 27M parameters and approximately 1000 training examples, Attractor Models achieve 91.4\% accuracy on Sudoku-Extreme and 93.1\% on Maze-Hard. In this regime, standard Transformers as well as proprietary frontier models such as DeepSeek R1, Claude, and o3-mini fail completely at 0\%, while specialized recursive architectures underperform our model and collapse when scaled. Attractor Models, in contrast, improve with scale.

%% file: sections/method.tex
\section{Background: Looped Architectures}

We begin with background on looped architectures. Let \(x=(x_1,\ldots,x_n)\in\mathcal V^n\) be an input sequence over vocabulary \(\mathcal V\), and let \(d\) denote the model width. Looped models can be written as a composition of three units: a {prelude} unit $\tilde x = \cP(x) \in\mathbb R^{n\times d}$, which produces an input representation $\tilde x \in \mathbb{R}^{n \times d}$; a weight-tied {recurrent} unit $h_{t+1} = \cR(h_t, \tilde x)$, which is applied repeatedly to a latent state \(h_t\in\mathbb R^{n\times d}\) for $T$ steps; and a {coda} unit, which maps the final latent state to output probabilities \(p = \cC(h_T) \in\Delta(\mathcal V)^n\). Importantly, looped architectures commonly initialize the latent state at an \emph{uninformative value}, such as $h_0=0$ or Gaussian noise $h_0 \sim \cN(0, \sigma^2 I)$~\cite{geiping, parcae, deq}. Furthermore, the recurrent step may use the input representation only at the first step~\cite{looplm}. or at every recurrent step~\cite{geiping, parcae}. Such injection may be through addition or concatenation with the recurrent state.

Models such as Parcae \cite{parcae}, Huggin \cite{geiping}, and Ouro \cite{looplm} differ mainly in how they train, stop, or scale this looped architecture. In particular, the recurrence depth $T$ is a central design choice in these models. It may be fixed \cite{trm}, sampled during training \cite{geiping, mcleish2025teachingpretrainedlanguagemodels, parcae}, or determined by an auxiliary halting mechanism \cite{mor, looplm}. Training then minimizes an objective averaged over both the data distribution and the chosen recurrence-depth mechanism, typically by backpropagation through depth. Consequently, both training cost and gradient memory are tied to the number of recurrent steps. Furthermore, changing $T$ at inference introduces a train--test mismatch, since the model is evaluated under a different computation graph than the one used during training, leading to degraded performance.

\section{Solve the Loop with Attractor Models}
\label{sec:method}

As discussed in the previous section, standard looped language models~\cite{looplm, parcae} use weight sharing to recurrently refine a hidden state that is initialized from an uninformative value and based on input embeddings. Predictions are then read out after a finite number of loops~\cite{parcae}, or once an auxiliary halting head becomes confident~\cite{act, looplm, parcae}. This design carries three drawbacks: the loop count must be chosen at training time, training memory grows linearly in the number of loops, and accuracy degrades when more loops are run at inference than were seen during training~\cite{looplm}. As a result, recurrence often comes with unstable training, growing memory requirements, and large sequential compute, in some cases approaching the costs of training non-recurrent models ten times larger~\cite{geiping}.

Recent mechanistic analyses of looped language models~\cite{mech} reveals that for the vast majority of tokens, the recurrent trajectory eventually converges to a fixed point. This suggests that the weight-tied recurrent modules are often approximating an underlying fixed-point computation, doing so through the recursive application of the weight-tied block truncated after $T$ steps. This observation motivates the design of Attractor Models, which we subsequently describe.

\begin{figure}[t]
  \centering
  \begin{subfigure}[t]{0.47\textwidth}
    \centering
    \includegraphics[
      width=\linewidth,
      trim={3.22cm 3.2cm 4.5cm 1.24cm},
      clip
    ]{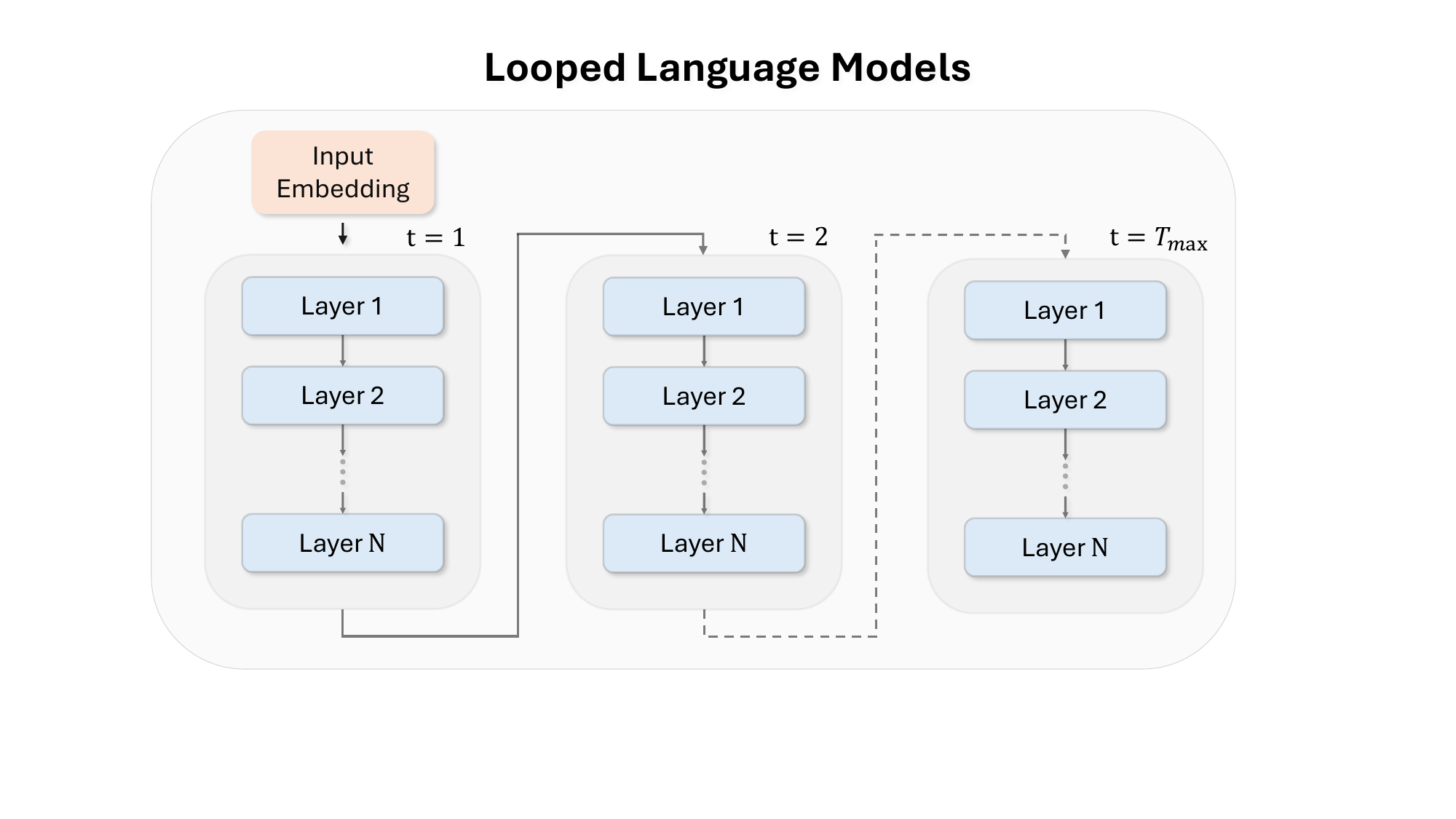}
    \label{fig:fig2}
  \end{subfigure}
  \hfill
  \begin{subfigure}[t]{0.50\textwidth}
    \centering
    \includegraphics[
      width=\linewidth,
      trim={4.2cm 2.8cm 1.8cm 1.45cm},
      clip
    ]{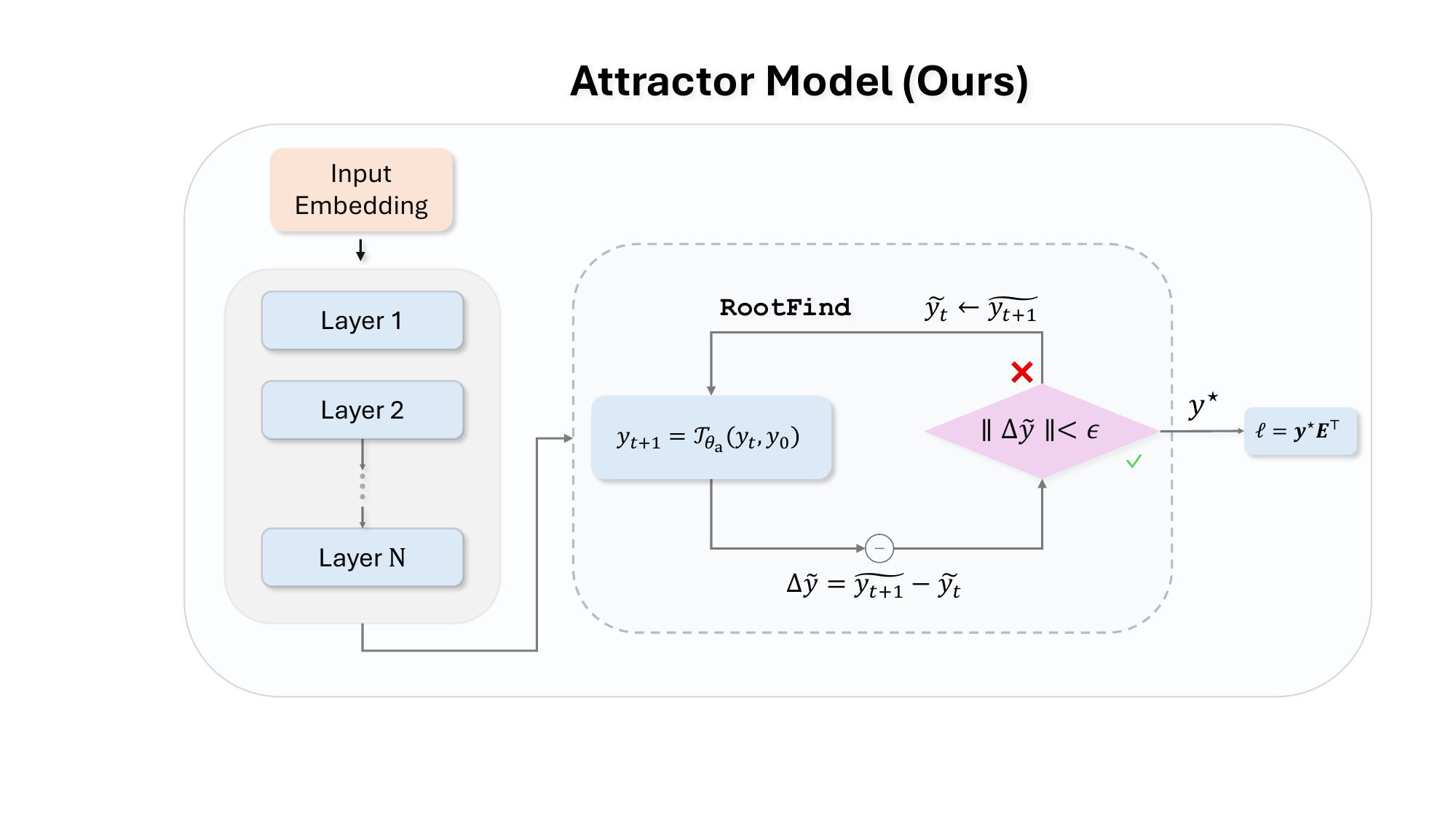}
    \label{fig:fig3}
  \end{subfigure}
  \caption{\textbf{Comparison of looped language models vs. Attractor Models.} Looped language models repeatedly apply a shared block for a finite number of steps before decoding the final state. In contrast, Attractor Models use a backbone to produce an initial output embedding, then refine it with an attractor module until the fixed-point residual is small, and decode the resulting approximate equilibrium.}
  \label{fig:side-by-side}
\end{figure}

\subsection{Attractor Model: Backbone and Attractor Modules}

Motivated by the fixed-point behavior observed in looped models, we model recurrent refinement as an attractor.  Rather than training a model to produce good predictions after a prescribed number of recurrent steps, we define the output as the equilibrium of the refinement process. Attractor Models consist of two modules: the \textit{backbone module} (typically a larger Transformer network) first proposes a meaningful initial output embedding, and the \textit{attractor module} (typically a smaller Transformer-based network) then refines this proposal until convergence. This makes the number of refinement steps a solver choice rather than a fixed architectural choice.

We first start by mapping the inputs $x$ into input embeddings $\tilde x = E(x) \in \mathbb{R}^{n \times d}$, where $E$ denotes the tied embedding/unembedding. Then, the input embedding is processed by the backbone and attractor modules as described below.

\textbf{The backbone module proposes an initial ``guess '' output embedding.} The backbone module $\mathcal T_{\theta_b}$ maps the input embeddings to an initial proposal:
\begin{align}
    \tilde y_0 = \mathcal T_{\theta_b}(\tilde x), \;\; \text{where} \;\; \tilde x = E(x).
    \label{eq:backbone_module}
\end{align}
We use $\tilde y_0$ as an initialization for the attractor module. Instead of initializing the loop from zero, noise, or an input-side representation, Attractor Models initialize the recurrent computation from a state that is already a coherent prediction embedding. In practice, $\mathcal T_{\theta_b}$ is a relatively high-capacity causal Transformer, so the refinement begins near a meaningful initialization rather than 0. We find that this makes training our method stable compared to DEQ, which experiences a blow-up in the number of iterations used later in training; whereas our method stabilizes later in training (c.f.~ Figure~\ref{fig:convergence_behavior}(b)).

\begin{figure}
    \centering
    \includegraphics[
        width=\linewidth,
        trim=0.36cm 5.75cm 1.83cm 5.37cm,
        clip
    ]{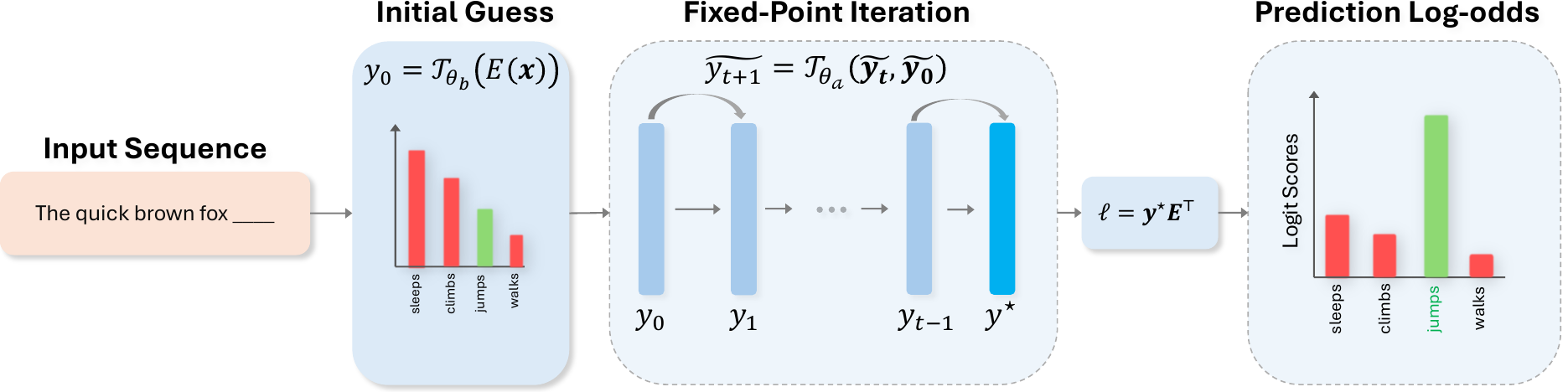}
    \caption{\textbf{Overview of Attractor Models.} The backbone maps the input embeddings to an output-side proposal $\tilde y_0=\cT_{\theta_b}(E(x))$. The attractor module then refines this proposal through the proposal conditioned iteration $\tilde y_{t+1}=\cT_{\theta_a}(\tilde y_t,\tilde y_0)$ until convergence to an approximate equilibrium  $\tilde y^\star$ or a maximum number of solver steps. The equilibrium is decoded with the tied unembedding $\tilde y^\star E^\top$.}
    \label{fig:overview}
\end{figure}

\textbf{The attractor module refines the output embedding.}
The attractor module is a separate weight-tied refinement network $\mathcal T_{\theta_a}$. Starting from the backbone proposal $\tilde y_0$, it repeatedly refines the output embedding according to
\begin{align}\label{eq:refinement_def}
    \tilde y_{t+1} = \cT_{\theta_a}(\tilde y_t, \tilde y_0), \;\; \text{where} \;\; \tilde{y}_0 = \cT_{\theta_b}(\tilde x).
\end{align}
Here, we persistently inject the initial guess $\tilde y_0$ at every refinement step. This persistent injection keeps the attractor proposal-dependent and prevents it from collapsing to a proposal-independent fixed point. We ablate this conditioning mechanism in Section~\ref{sec:experiments}. Importantly, we \textit{warm-start} the attractor module at an informative proposal $\tilde y_0$, in contrast to existing work that initialize the recurrent state at uninformative values such as zero or Gaussian noise \cite{geiping, parcae}; see Table \ref{tab:ablation_init} for a comparison.

Rather than rolling out recurrent steps to reach a fixed point, we directly solve for the convergence:
\begin{align}\label{eq:attractor_module}
    \cA_{\theta_a}(\tilde y^\star, \tilde y_0) \coloneqq \cT_{\theta_a}(\tilde y^\star,\tilde y_0) - \tilde  y^\star = 0 
    \; \Rightarrow \; \tilde y^{\star} \;=\;
    \texttt{RootFind}\!\left(\cA_{\theta_a}(\cdot, \tilde y_0); y_0\right).
\end{align}
In the forward pass, we compute this equilibrium with a root finder initialized at the backbone proposal. In our implementation, the \texttt{RootFind} algorithm uses Anderson acceleration, which combines a small window of past iterates and residuals to reach the fixed point faster than plain recursion. The solver exits when ${\|\cA_{\theta_a}(\tilde y_t,\tilde y_0)\|_2}/{\|\tilde y_t\|_2}<\varepsilon$ or after $T_{\max}$ steps. Thus, the computation is controlled by the convergence of the residual rather than by a learned halting head or a preset loop count. In contrast to fixed unrolling, the number of refinement steps can therefore vary at inference time without changing the model. Finally, the equilibrium embedding is decoded with the tied unembedding. 

Parameters of the Attractor Models consist of the tied embedding/unembedding matrices, the backbone module, and parameters of the attractor module: $\theta \coloneqq (\theta_a, \theta_b, E)$. Compared to looped models, Attractor Models change both the starting point and the endpoint of recurrence: we initialize the loop from the output guess from the backbone network $\tilde y_0$, and the decoded state is the attractor $\tilde y^\star$ rather than a finite unroll.

\subsection{Training and Inference of Attractor Models}

We now describe the training procedure for Attractor Models. We first explain how to differentiate through the fixed-point solver using implicit differentiation, and then show how the model is optimized with the standard cross-entropy language-modeling loss applied to the output $y^{\star}$.

\textbf{Backward pass and implicit differentiation.} Because Attractor Models define the output embedding as the solution to a fixed-point equation, we differentiate through the equilibrium using the implicit function theorem \cite{krantz2002implicit}. Let $\mathcal L$ denote the training loss and let $v=\partial \mathcal L/\partial \tilde y^\star$. Applying the implicit function theorem to $\cA_{\theta_a}(\tilde y^\star,\tilde y_0)=0$ gives
\begin{align}\label{eq:implicit_gradient}
    \frac{\partial \mathcal L}{\partial \theta}
    =
    u^\top
    \frac{\partial \cT_{\theta_a}(\tilde y^\star,\tilde y_0)}{\partial \theta},
    \quad
    u
    =
    \left(I-J_{\tilde y}^{\top}\right)^{-1}v, \quad \text{where} \quad J_{\tilde y}
    =
    \left.
    \frac{\partial \cT_{\theta_a}(\tilde y,\tilde y_0)}
    {\partial \tilde y}
    \right|_{\tilde y=\tilde y^\star}.
\end{align}
The derivative with respect to $\theta=(\theta_a,\theta_b,E)$ includes the direct dependence on the attractor parameters $\theta_a$, as well as the dependence on the initialization $\tilde y_0=\cT_{\theta_b}(E(x))$ through the backbone parameters $\theta_b$ and the tied embedding parameters $E$.

Following prior work on implicit models~\cite{phantomgradient, jfb}, we use the one-step approximation $u\approx v$. This avoids the extra linear solve for $u$ and reduces the backward pass to one vector--Jacobian product through $\cA_{\theta_a}$. Since we do not backpropagate through every solver step, memory in the attractor block does not grow with the number of forward iterations. In Section~\ref{sec:experiments}, we show that the Anderson solver yields only marginal quality gains, whereas the one-step approximation enables substantially cheaper training.

\textit{Remark.} Attractor Models are implicit equilibrium models in the spirit of DEQs~\cite{deq}, but the equilibrium plays a different role. Classical DEQs replace the prediction network with a hidden-state equilibrium $z^{\star}$ (single-layer), decoded with a separate output head. We instead keep a standard causal Transformer backbone and add an equilibrium refinement block on top of its prediction state. The fixed point lives directly in the tied embedding space, so every iterate $\{y_0,y_1,\ldots,y^{\star}\}$ is already a representation in the output space that can be decoded. This gives three practical differences: (i) the solver is initialized from a semantically meaningful proposal $\tilde y_0$ rather than from an uninformative state such as zero (as in DEQ), (ii) inference can stop according to a residual tolerance $\varepsilon$ rather than a fixed depth or learned halting head, and (iii) DEQ shows that scaling the number of DEQ blocks can harm the performance of their method, whereas we allow for an arbitrary depth backbone transformer and show that we can use a variable number of solver blocks.

\begin{algorithm}[!h]
\caption{Training Attractor Models.}
\label{alg:method}
\begin{algorithmic}[1]
\Require input tokens $x$, parameters $\theta=(\theta_a,\theta_b,E)$, tolerance $\varepsilon$, max iterations $T_{\max}$
\Statex \textbf{// Forward pass}
\State $\tilde x \gets E(x)$ \Comment{tied token embedding}
\State $\tilde y_0 \gets \cT_{\theta_b}(\tilde x)$ \Comment{backbone proposal / initialization}
\State Define $\cA_{\theta_a}(\tilde y;\tilde y_0) \gets \cT_{\theta_a}(\tilde y,\tilde y_0)-\tilde y$
\State $\tilde y^\star \gets \texttt{RootFind}\!\left(\cA_{\theta_a}(\cdot;\tilde y_0);\; \tilde y_0,\; \varepsilon,\; T_{\max}\right)$
\State $p \gets \mathsf{Softmax}(\tilde y^\star E^\top)$ \Comment{tied unembedding and softmax}

\Statex \textbf{// Backward pass, given } $v=\partial \mathcal L/\partial \tilde y^\star$
\State solve $\left(I - J_{\tilde y}^{\top}\right) u = v$ via $\texttt{RootFind}\!\left(u' \mapsto (I - J_{\tilde y}^{\top})\,u' - v;\; u_0=v\right)$
\State $\partial \mathcal L/\partial \theta \gets
u^\top
{\partial \cT_{\theta_a}(\tilde y^\star,\tilde y_0)}/{\partial \theta}$
\end{algorithmic}
\end{algorithm}

\textbf{Training objective and inference.}  We train Attractor Models with the standard next-token prediction cross-entropy objective applied to the fixed-point output $y^{\star}$. Inference reuses the same equilibrium computation. Given an input sequence $x$, the backbone first produces the proposal $\tilde y_0=\cT_{\theta_b}(E(x))$, the attractor solver computes $\tilde y^\star$, and the tied unembedding decodes $\tilde y^\star$ into next-token probabilities. Peak memory is bounded by a single forward through the attractor module and standard KV-caching applies in the backbone. In principle, the solver tolerance $\varepsilon$ and maximum iteration budget $T_{\max}$ are inference-time hyperparameters: they can be adjusted without retraining the model, turning test-time computation into a budget for approaching the learned attractor. Interestingly, however, we find that trained Attractor Models often require very little test-time refinement, as we describe in the next section.

\subsection{Equilibrium Internalization and Stability in Attractor Models}

Although Attractor Models define predictions through the equilibrium $\tilde y^\star$, we observe a surprising phenomenon: after training, the backbone proposal $\tilde y_0$ often lies close to the equilibrium (c.f.~Figures \ref{fig:convergence_behavior} and \ref{fig:accuracy_vs_T}). We refer to this phenomenon as \emph{equilibrium internalization}. Intuitively, the attractor module appears to act as a moving teacher for the backbone, resulting in a form of automatic curriculum. Early in training, the proposal $\tilde y_0$ may be far from a good prediction, and the solver must perform nontrivial refinement to reach $\tilde y^\star$. Since $\tilde y_0$ and $\tilde y^\star$ live in the same tied output-embedding space, gradients through the equilibrium also train the backbone proposal to move toward the state that the solver would have found. Thus, when the backbone is sufficiently expressive, much of the prediction work can be internalized into $\tilde y_0$, leaving the attractor to perform a stable refinement.

\textbf{Stability.}
Equilibrium training also biases the recurrent map toward convergent dynamics. The implicit gradient contains the inverse factor $(I-J_{\tilde y}^{\top})^{-1}$, which becomes ill-conditioned near non-contractive regimes. This creates a barrier against unstable fixed-point dynamics, unlike fixed-loop training, which can learn trajectories that are accurate only at a prescribed step count and fail under extra inference-time loops. We discuss this contrast in detail in Appendix~\ref{subsec:fall}. We refer to Appendix~\ref{sec:theory} for theoretical analysis of Attractor Models and comparison with finite-loop models.

%% file: sections/experiments.tex
\section{Experiments}
\label{sec:experiments}

We evaluate Attractor Models in two regimes. We first study language modeling across model sizes, comparing against parameter-matched Transformers and looped-LM baselines. We evaluate scaling behavior, downstream accuracy, and training efficiency. We also present results on hard reasoning tasks, where we test whether the same fixed-point refinement mechanism improves small models on problems that require iterative computation.

\subsection{Attractor Models Improve Large-Scale Language Modeling}\label{sec:experiments_lm}

\textbf{Setup.} We follow the \texttt{nanochat}~\cite{nanochat} pretraining recipe used by Parcae~\cite{parcae} for its main Transformer comparison, training on FineWeb-Edu~\cite{fineweb}. To ensure a fair comparison, all models are matched in parameter count and trained with the same data budget, optimizer, and learning-rate schedule as the Parcae baselines; the only change is the recurrent block. We compare at three scales: 140M, 370M, and 770M parameters. Parcae is a middle-looped language model with prelude, coda, and recurrent blocks. The stability in this model comes from a linear injection that bounds the spectral radius of the recurrence below one. In contrast, our architecture uses standard Transformer blocks followed by a fixed-point iteration block, and lets the solver itself control the effective depth. Detailed hyperparameter settings are in Appendix \ref{sec:hyperparam}.

\textbf{Parameter Scaling.} We demonstrate how large-scale pretraining scales with our model. We evaluate our method on 140M, 370M, and 770M parameters against a parameter-matched Transformer and {Parcae}, a looped-language model~\cite{parcae}. Our model improves monotonically with scale. Across all three sizes, our method achieves the best validation PPL, Lambada PPL, and CORE accuracy. These results show that our fixed-point refinement scales cleanly with model size, with especially large gains on Lambada, where iterative refinement substantially improves long-context prediction.

\begin{table}[!hb]
\centering
\caption{\textbf{Parameter scaling results on large-scale language modeling.} Our Attractor Model outperforms standard Transformers and looped model \textsc{Parcae} on nearly every benchmark. Our 770M model performs comparably to a standard Transformer with nearly twice as many parameters and trained on twice as many tokens. Arrows indicate the relative improvement over the parameter-matched standard Transformer baseline. Arrows indicate the relative improvement over the parameter-matched Transformer baseline.}
\label{tab:parcae_results}
\small
\setlength{\tabcolsep}{0pt}
\renewcommand{\arraystretch}{1.04}
\color{TableText}
\begin{tabular*}{\linewidth}{
    @{}
    p{0.07\linewidth}
    @{\hspace{0.25em}}
    l
    @{\extracolsep{\fill}}
    @{\hspace{0.25em}}
    c
    @{\hspace{0.25em}}
    c
    @{\hspace{0.25em}}
    c
    @{\hspace{0.25em}}
    c
    @{\hspace*{1.8em}}
    @{}
}
\arrayrulecolor{TableRule}
\textbf{Size} &
\textbf{Model} &
\textbf{Val. PPL} $\downarrow$ &
\textbf{Lambada PPL} $\downarrow$ &
\textbf{Core} $\uparrow$ &
\textbf{Core-Ext.} $\uparrow$ \\
\arrayrulecolor{TableRule}\specialrule{1pt}{1pt}{1pt}

\textcolor{TableAccent}{\textbf{140M}}
& Transformer & 21.48 & 127.39 & $13.00 \pm 0.15$ & $8.80 \pm 0.21$ \\
& \textsc{Parcae} & 19.06 & 80.64 & $14.04 \pm 0.20$ & $9.67 \pm 0.28$ \\
& \textcolor{TableAccent}{\textbf{Attractor Model (ours)}}
& \best{\textbf{18.30}}\makebox[0pt][l]{\hspace{0.12em}{\scriptsize\gaindown{14.8}}}
& \best{\textbf{68.02}}\makebox[0pt][l]{\hspace{0.12em}{\scriptsize\gaindown{46.6}}}
& \best{$\bm{14.59 \pm 0.11}$}\makebox[4pt][l]{\hspace{0.12em}{\scriptsize\gainup{12.2}}}
& \best{$\bm{10.03 \pm 0.05}$}\makebox[0pt][l]{\hspace{0.12em}{\scriptsize\gainup{14.0}}} \\

\grouprule

\textcolor{TableAccent}{\textbf{370M}}
& Transformer & 15.79 & 40.77 & $17.46 \pm 0.03$ & $11.71 \pm 0.22$ \\
& \textsc{Parcae}
& 14.49
& 32.74
& $20.00 \pm 0.06$
& \best{$\bm{12.75 \pm 0.31}$} \\
& \textcolor{TableAccent}{\textbf{Attractor Model (ours)}}
& \best{\textbf{14.03}}\makebox[0pt][l]{\hspace{0.12em}{\scriptsize\gaindown{11.1}}}
& \best{\textbf{27.14}}\makebox[0pt][l]{\hspace{0.12em}{\scriptsize\gaindown{33.4}}}
& \best{$\bm{20.24 \pm 0.09}$}\makebox[0pt][l]{\hspace{0.12em}{\scriptsize\gainup{15.9}}}
& $12.64 \pm 0.33$\makebox[0pt][l]{\hspace{0.12em}{\scriptsize\gainup{7.9}}} \\

\grouprule

\textcolor{TableAccent}{\textbf{770M}}
& Transformer & 13.08 & 22.37 & $22.42 \pm 0.20$ & $14.20 \pm 0.63$ \\
& \textsc{Parcae} & 12.49 & 19.71 & $25.07 \pm 0.33$ & $15.19 \pm 0.43$ \\
& \textcolor{TableAccent}{\textbf{Attractor Model (ours)}}
& \best{\textbf{12.09}}\makebox[0pt][l]{\hspace{0.12em}{\scriptsize\gaindown{7.6}}}
& \best{\textbf{15.21}}\makebox[0pt][l]{\hspace{0.12em}{\scriptsize\gaindown{32.0}}}
& \best{$\bm{26.83 \pm 0.29}$}\makebox[0pt][l]{\hspace{0.12em}{\scriptsize\gainup{19.7}}}
& \best{$\bm{15.42 \pm 0.51}$}\makebox[0pt][l]{\hspace{0.12em}{\scriptsize\gainup{8.6}}} \\

\grouprule

\textcolor{TableAccent}{\textbf{1.3B}}
& Transformer & 11.95 & 17.26 & $25.45 \pm 0.08$ & $15.90 \pm 0.23$ \\

\end{tabular*}
\end{table}

\textbf{Training efficiency: Lower FLOPs.}
 The one-step IFT backward pass keeps training memory constant in the number of solver iterations. The memory of standard looped language models like~\cite{parcae} scales linearly with the number of loops. The total training FLOPs (Figure~\ref{fig:train_flops}) follow the same trend: although every step of our recurrent block costs roughly the same as Parcae's, our solver typically converges below $\varepsilon$ in well under $T_{\max}$ steps, so the realized depth during training is lower despite identical $T_{\max}$, yielding 25--31\% lower training FLOPs across scales.

\begin{figure}[!ht]
    \centering
    \includegraphics[width=0.75\linewidth,, trim=0 5cm 2.5cm 2.6cm, clip]{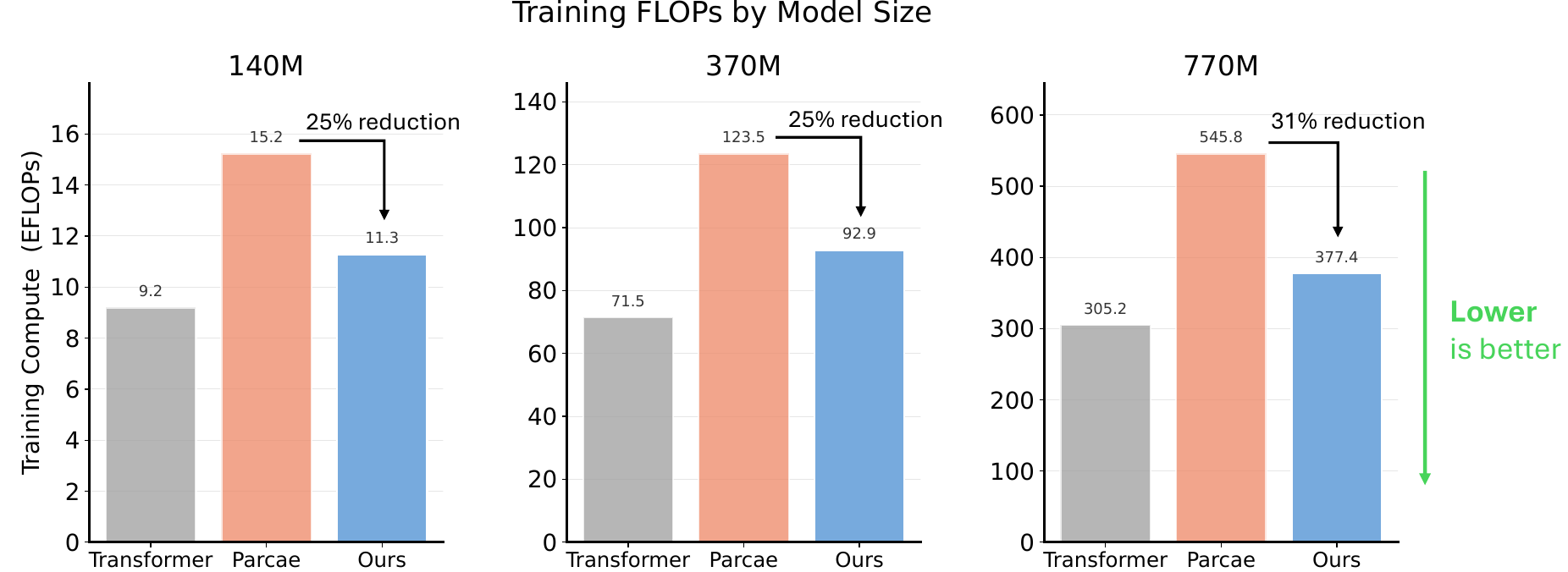}

    \caption{\textbf{Training-time efficiency.}
    Despite being a looped architecture, our method uses 25--31\% fewer FLOPs than Parcae because the fixed-point solver often converges before $T_{\max}$ and the backward pass uses the one-step implicit-gradient approximation.}
    \label{fig:train_flops}
\end{figure}

\noindent
\begin{minipage}[t]{0.43\linewidth}
\textbf{Training efficiency: O(1) memory.} An additional advantage of our method is that the training memory required stays constant with the number of iterations (Figure~\ref{fig:memory_vs_T_large}). This is because our implicit backward pass does not need to store the intermediate activations from every recurrent step. In contrast, standard looped language models must backpropagate through each unrolled iteration, causing memory usage to grow linearly with the number of loops.
\end{minipage}
\hfill
\begin{minipage}[t]{0.46\linewidth}
\vspace{-1em}
\centering
\includegraphics[width=0.8\linewidth]{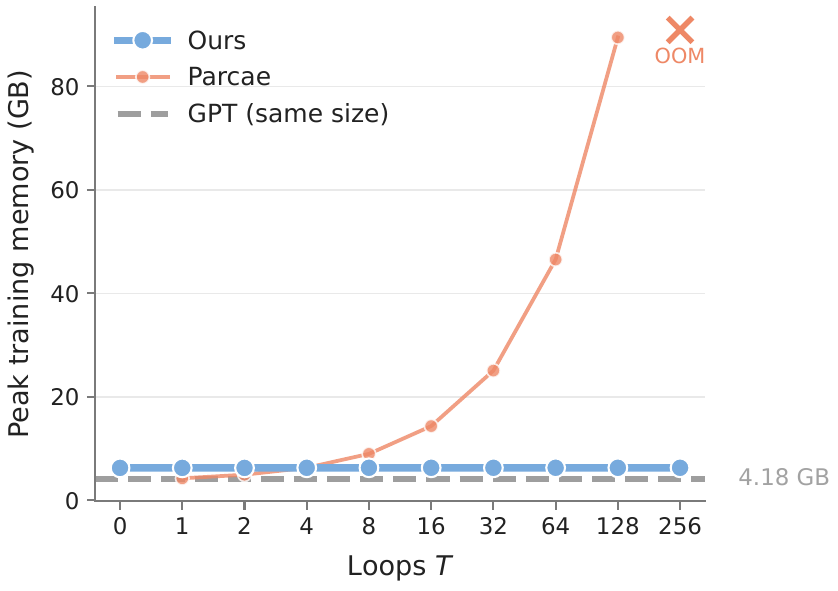}
\captionof{figure}{\textbf{Peak training memory versus recurrent depth.} Attractor Models keep memory nearly constant, while Parcae's memory grows with the number of loops.}
\label{fig:memory_vs_T_large}
\end{minipage}

\begin{greenAIbox}{Result: Attractor Models deliver Pareto improvement on large-scale language modeling.}
Across model sizes, Attractor Models outperform parameter-matched Transformers and stable looped baselines on perplexity and downstream accuracy, while using substantially less training compute and memory than looped models. Notably, the 770M Attractor Model reaches the quality regime of a 1.3B Transformer trained on twice as many tokens.
\end{greenAIbox}

\subsection{Attractor Models Lead to Significant Gains in Hard Reasoning Tasks}
\label{sec:exp-algo}

Beyond language modeling, we train and evaluate Attractor Models on Sudoku-Extreme and Maze-Hard from~\cite{hrm,trm}: two challenging reasoning benchmarks, where non-recurrent Transformers and most frontier LLMs fail.

\textbf{Setup.} We train small models with approximately 1000 training examples for each task and require predicting the full output grid in a single direct forward pass (no autoregressive decoding). We follow the TRM training protocol~\cite{trm}, using deep-supervision steps.

\textbf{Method.} TRM carries two latents across deep-supervision steps, a current answer $y$ and a reasoning state $z$, and applies a tiny two-layer network $T(n{+}1)$ times per step to update them. We keep this protocol except for the inner update. Specifically, instead of unrolling $T(n{+}1)$ applications, we solve directly for the fixed point of the $(y, z)$ update with our solver. The initialization is handled by deep supervision itself: the previous step's $(y, z)$ initializes the solver at the next step, with a learned embedding at step zero, so we do not use a separate backbone.

We present our results in Table~\ref{tab:direct_prediction_results}. The fixed-depth Transformer fails on both tasks. HRM (27M) achieves 55.0\% and 74.5\%, respectively on Sudoku-Extreme and Maze-Hard. TRM is the strongest tiny baseline at 7M, achieving 74.7\% and 85.3\%, but (counter to the goal of scaling) collapses to 0\% on both tasks when scaled to 27M parameters. Our model scales naturally with parameter count. We attribute the difference to the explicit fixed-point objective, which appears to provide regularization that bare iterative refinement lacks at higher capacity.

For the backward pass, we use the phantom-gradient scheme~\cite{phantomgradient} rather than the one-step approximation used in the language-modeling experiments. This choice is important in the small-data reasoning regime: with only $\sim$1{,}000 training examples and much smaller networks, the solver dynamics are more sensitive, and the one-step surrogate can provide too crude a training signal. This is consistent with TRM, which reports that replacing its backward pass with a one-step approximation reduces Sudoku-Extreme accuracy from 87.4\% to 56.5\%~\cite{trm}.

\begin{table}[h]
\centering
\caption{\textbf{Small model and sample training results on Sudoku-Extreme and Maze-Hard.} Our Attractor Models improve when scaling the parameter count, compared to standard tiny recursive models which degrade.}
\label{tab:direct_prediction_results}

\arrayrulecolor{TableRule}
\begin{tabular}{@{} l c c c @{}}
\HeaderStrut
\textbf{Method} & \textbf{\# Parameters} & \textbf{Sudoku-Extreme} $\uparrow$ & \textbf{Maze-Hard} $\uparrow$ \\
\specialrule{1pt}{0pt}{0pt}

Deepseek R1 & 671B & 0.0\% & 0.0\% \\
Claude 3.7 & ? & 0.0\% & 0.0\% \\
O3-mini-high & ? & 0.0\% & 0.0\% \\
\grouprule

Transformer & 27M & 0.0\% & 0.0\% \\
\grouprule

HRM & 27M & 55.0\% & 74.5\% \\
TRM & 7M & 74.7\% & 85.3\% \\
TRM & 27M & 0.0\% & 0.0\% \\
\grouprule

\textcolor{TableAccent}{\textbf{Attractor Model (ours)}} & 7M & 54.3\% & 46.7\% \\
\textcolor{TableAccent}{\textbf{Attractor Model (ours)}} & 27M & \best{\textbf{91.4\%}} & \best{\textbf{93.1\%}} \\

\end{tabular}
\end{table}

\noindent
This setup is still an instance of our Attractor Model framework, where an output-space representation is iteratively refined to an equilibrium, decoded, and differentiated through implicitly. The only difference is the initialization mechanism. Rather than a Transformer backbone producing $\tilde y_0$ from the input, deep supervision supplies the initialization. Specifically, the previous supervision step's $(y, z)$ initializes the solver at the next step, with a learned embedding at step zero.

\begin{greenAIbox}{Result: Attractor Models scale where recursive reasoners collapse.}
In hard reasoning tasks, existing recursive architectures can degrade as model capacity increases. Attractor Models avoid this failure mode. With only 27M parameters and $\sim$1{,}000 training examples, our model reaches 91.4\% accuracy on Sudoku-Extreme and 93.1\% on Maze-Hard, outperforming both frontier LLMs and specialized recursive baselines.
\end{greenAIbox}

\subsection{Equilibrium Internalization and Test-Time Behavior}
\label{sec:exp-test}

\textbf{Fixed-point convergence.}
Figure~\ref{fig:convergence_behavior} visualizes convergence in two complementary ways. 
First, we project the trajectory of the final-position representation onto its first two principal components over 16 iterations. 
Our method rapidly contracts to a fixed point: iterations 8--16 collapse onto a single attractor. Parcae also moves toward a stable state, but does so more slowly; its trajectory remains noisier and continues to drift through later recurrences. 
This suggests that while finite-loop language models can exhibit fixed-point-like behavior, explicitly training the loop as a fixed-point problem produces faster and cleaner convergence. Second, we track the number of solver iterations required during training. 
The DEQ baseline requires increasingly many iterations as optimization progresses, consistent with the observations in the original work \cite{deq}. In contrast, Attractor Models rapidly converge to the minimum iteration count and remain stable. This behavior provides evidence for {equilibrium internalization}, where optimization shifts work from the iterative solver into the backbone proposal.

\begin{figure}[!h]
    \centering

    \begin{subfigure}[t]{0.48\linewidth}
        \centering
        \includegraphics[width=\linewidth]{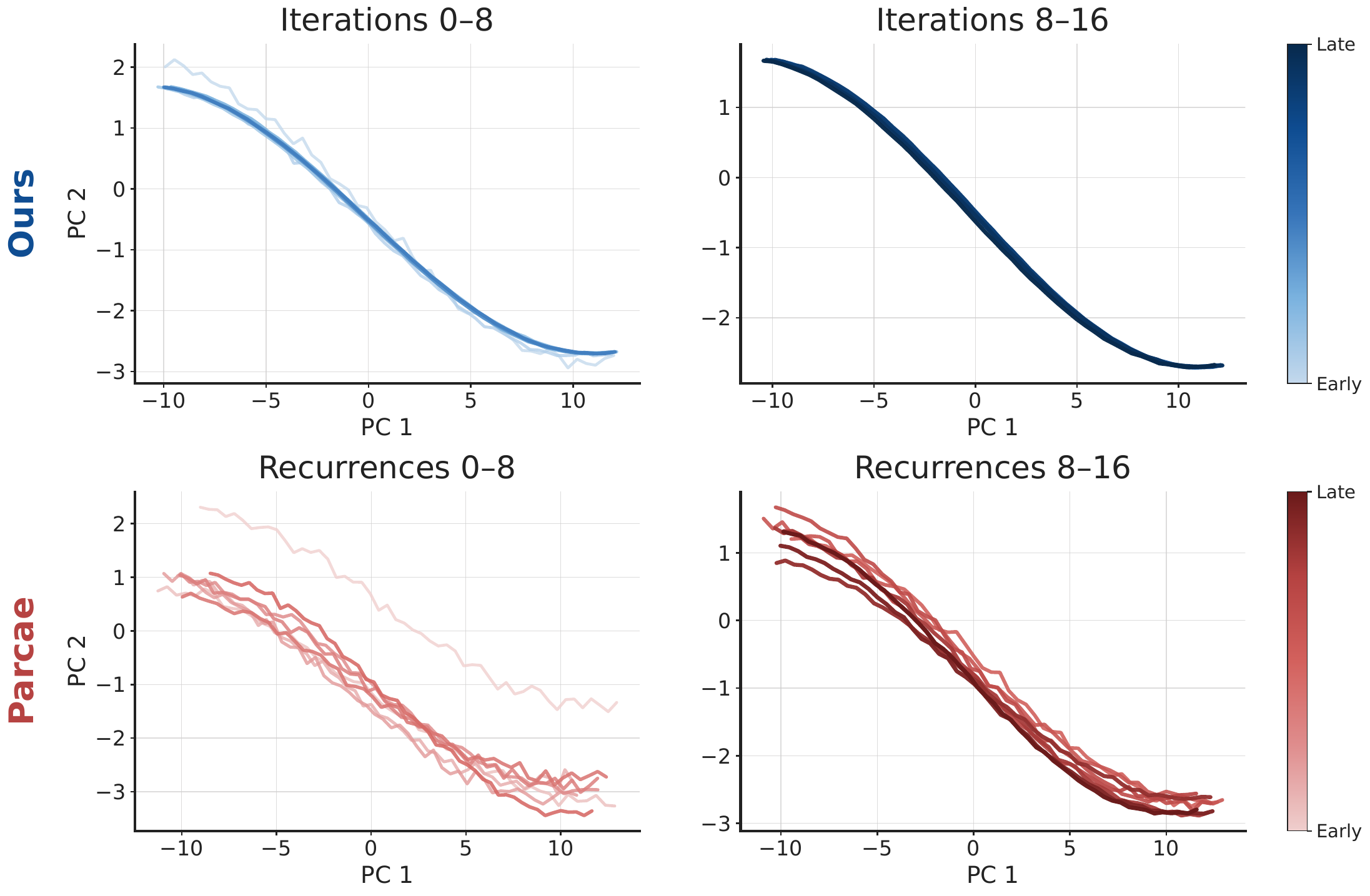}
        \caption{PCA trajectory over 16 iterations.}
        \label{fig:trajectory}
    \end{subfigure}
    \hfill
    \begin{subfigure}[t]{0.48\linewidth}
        \centering
        \includegraphics[width=\linewidth]{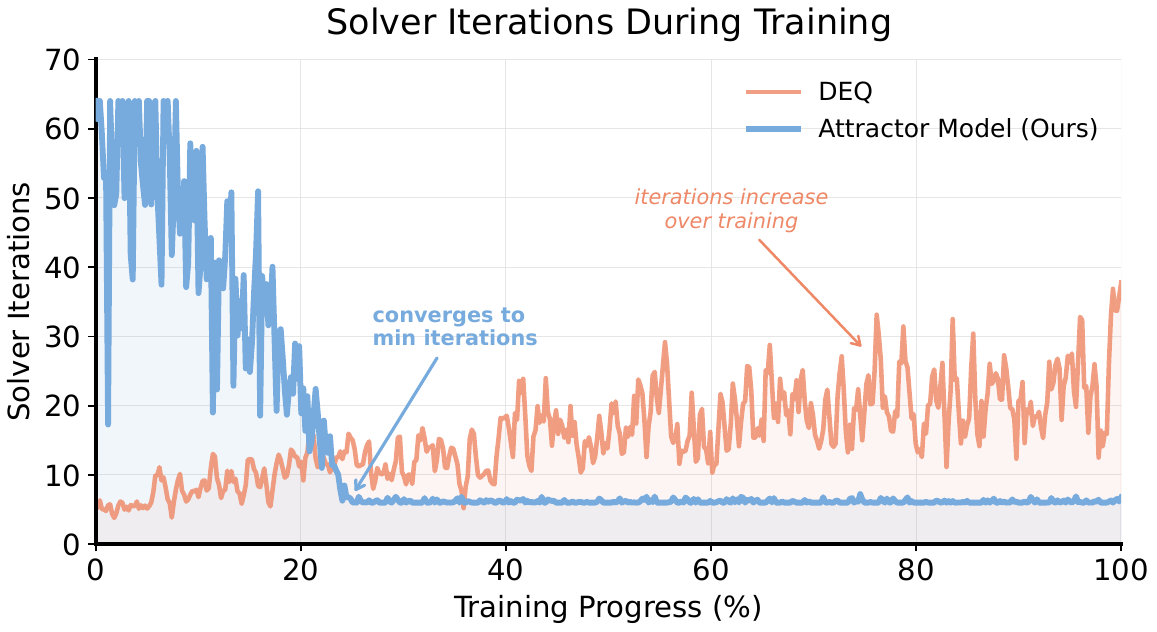}
        \caption{Solver iterations during training.}
        \label{fig:iter_comparison}
    \end{subfigure}

    \caption{\textbf{Convergence behavior of Attractor Models compared to existing methods.}
     \textbf{Left:} PCA projection of the residual stream at the final sequence position over 16 iterations. \textbf{Right:} DEQ baseline requires increasingly many solver iterations during training, whereas Attractor Models rapidly settles to the minimum iteration count and remains stable.}
    \label{fig:convergence_behavior}
\end{figure}

\textbf{Test-time iterations vs.\ quality.}
In Figure~\ref{fig:accuracy_vs_T}, we report validation perplexity, CORE accuracy, and CORE-Extended accuracy as we vary the number of inference iterations $T$ while holding training fixed. For Parcae, quality improves monotonically from $T=1$ until it plateaus near $T=8$, indicating that the model relies on repeated test-time refinement. In contrast, our method reaches peak performance at $T=1$ at every scale, and even $T=0$ (decoding the backbone proposal $\tilde y_0$ directly through the tied unembedding) is already at or near the converged value.

We attribute this behavior to \emph{equilibrium internalization}: during training, the backbone learns to produce a proposal $\tilde y_0$ that already lies close to the fixed point that the solver would otherwise compute through iteration. Thus, the iterative block still shapes the learned representation during training, but at inference time the model has largely internalized the result of that refinement into its backbone initialization. Strikingly, the backbone-only readout at $T=0$ is already stronger than larger standard Transformers trained without attractor assistance, and matches or exceeds finite-loop baselines that require many test-time recurrences. Concretely, our 140M model at $T=0$ (no solver) matches or improves Parcae 140M at $T=8$ recurrent steps; the same pattern holds at 370M and 770M.

\begin{figure}[!h]
    \centering
    \includegraphics[width=\linewidth]{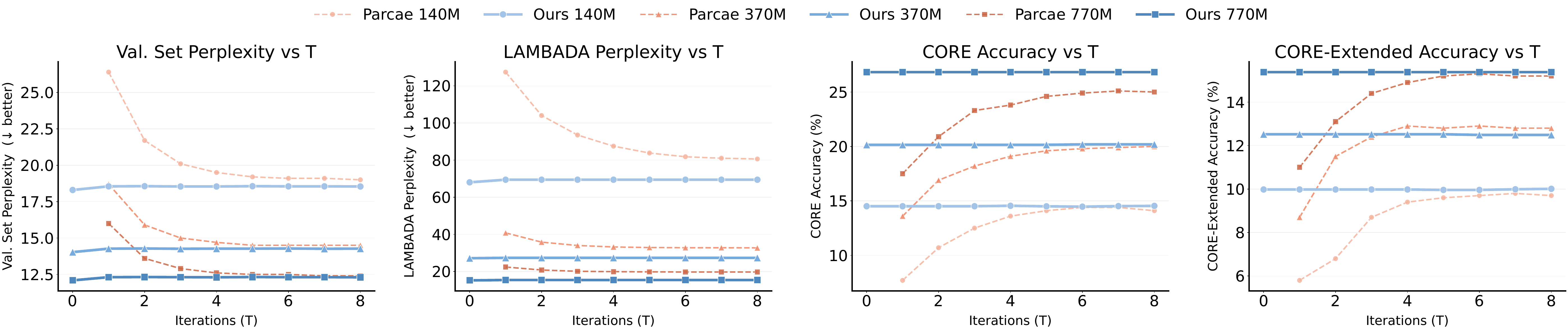}
    \caption{\textbf{Validation perplexity on FineWeb-Edu, Lambada, CORE accuracy, and CORE-Extended accuracy as a function of test-time iterations  $\; T$}. Parcae's quality improves monotonically up to $T \approx \mu_\text{rec} = 8$. Our method is essentially converged at $T=1$ (and at $T=0$ for the 770M model), because the solver is initialized at $\tilde y_0$, which is itself a coherent prediction.}
    \label{fig:accuracy_vs_T}
\end{figure}

\begin{AIbox}{Takeaway: Attractor Models internalize equilibrium.}
Attractor Models learn to initialize a prediction near the fixed point, so inference is effectively converged before any explicit refinement step. As a result, they preserve the benefits of iterative computation during training while avoiding the need for many sequential iterations at test time, unlike standard looped models whose quality depends on repeated inference-time updates.
\end{AIbox}

\subsection{Ablations}
\label{sec:exp-ablation}

For the ablations, we use a 60.3M-parameter version of our model trained on 1B tokens of FineWeb-Edu under the same \texttt{nanochat} configuration. All ablation runs share the same data stream, tokenizer, and optimizer; only the ablated component varies.

We isolate the contribution of (i) the location of the equilibrium and its initialization (Table~\ref{tab:ablation_deq}), (ii) the additive injection of $c$ (Table~\ref{tab:ablation_injection}), (iii) and the one-step backward approximation (Table~\ref{tab:ablation_backward}). All ablations train a 60.3M-parameter model on 1B tokens of FineWeb-Edu under the same \texttt{nanochat} setup, with all other hyperparameters fixed.

\textbf{Comparison to DEQ.}
Table~\ref{tab:ablation_deq} compares Attractor Models against a
parameter-matched DEQ~\cite{deq}. The DEQ baseline solves for a hidden-state
equilibrium initialized from an uninformative state and decodes it with a
separate output head. Tying the unembedding closes part of the gap
(Val.\ PPL $42.18 \to 38.74$), but the largest improvement comes from matching
the equilibrium to the design of Attractor Models: the fixed point is placed
directly in the tied output-embedding space and the root finder is initialized
from the backbone proposal $\tilde y_0=\cT_{\theta_b}(E(x))$ ($\to 34.05$).

This proposal gives the solver a meaningful, input-dependent starting point
rather than requiring the recurrent block to construct the representation from
scratch. As a result, the Attractor Model reaches the same residual tolerance
in $1.7\times$ fewer iterations while also achieving lower perplexity. We view
this as evidence of \emph{equilibrium internalization}: the backbone learns to
produce an output-side proposal $\tilde y_0$ that already lies close to the
eventual equilibrium $\tilde y^\star$, making the subsequent attractor
refinement both easier and faster.

\begin{table}[!ht]
\centering
\caption{\textbf{Comparison to DEQ at matched parameter count.}}
\label{tab:ablation_deq}

\small
\setlength{\tabcolsep}{3pt}
\renewcommand{\arraystretch}{1.05}
\setlength{\abovecaptionskip}{0pt}
\setlength{\belowcaptionskip}{2pt}

\color{TableText}

\begin{tabularx}{\linewidth}{
  @{}
  >{\raggedright\arraybackslash}X
  c
  >{\raggedright\arraybackslash}X
  c
  c
  c
  @{}
}
\arrayrulecolor{TableRule}

\textbf{Method} &
\textbf{Size} &
\textbf{Equilibrium} &
\textbf{Val. PPL} $\downarrow$ &
\textbf{Avg. Iters} $\downarrow$ &
\textbf{Core} $\uparrow$ \\
\specialrule{1pt}{1pt}{1pt}

DEQ~\cite{deq}
& 60.3M
& hidden ($z_0=0$, sep. head)
& 42.18
& 14.6
& $5.21 \pm 0.14$ \\

DEQ + tied unemb.
& 60.3M
& hidden ($z_0=0$)
& 38.74
& 13.9
& $5.83 \pm 0.12$ \\

\grouprule

\textcolor{TableAccent}{\textbf{Attractor Model}}
& 60.3M
& output emb. ($\tilde y_0=\cT_{\theta_b}(E(x))$)
& \best{\textbf{34.05}}
& \best{\textbf{8.4}}
& \best{$\bm{6.74 \pm 0.10}$} \\

\end{tabularx}

\end{table}

\textbf{Proposal injection.}
Table~\ref{tab:ablation_injection} ablates how the backbone proposal
$\tilde y_0$ is provided to the attractor. We use \emph{proposal injection} to denote how $\tilde y_0$ enters the recurrent refinement map, and \emph{persistent injection} to denote providing $\tilde y_0$ at every refinement step rather than only through the initial state.

When the proposal is used only for initialization, $\tilde y_{t=0}=\tilde y_0$ and $\tilde y_{t+1}=\cT_{\theta_a}(\tilde y_t)$, the recurrent update no longer depends on the input after the first state. Consequently, the fixed point can become proposal-independent, and only 12.4\% of validation tokens converge within $T_{\max}$. Persistent injection by concatenation,
$\tilde y_{t+1}=\cT_{\theta_a}([\tilde y_t;\tilde y_0])$, restores proposal-dependence and recovers much of the quality, but makes the refinement problem harder: convergence is slower (11.2 vs.\ 8.4 average iterations) and
perplexity is worse (36.81 vs.\ 34.05). Additive proposal injection,
$\tilde y_{t+1}=\cT_{\theta_a}(\tilde y_t,\tilde y_0)$, provides the backbone
proposal at every step while keeping the refinement map simple and
well-conditioned, yielding the best results across all three metrics.

\begin{table}[!ht]
\caption{\textbf{Effect of input injection on convergence and quality.}}
\centering
\small
\setlength{\tabcolsep}{5pt}
\renewcommand{\arraystretch}{1.22}
\color{TableText}
\begin{tabularx}{\linewidth}{@{\hspace{5pt}} l Y Y Y @{\hspace{5pt}}}
\arrayrulecolor{TableRule}
\HeaderStrut
\textbf{Injection of $c$} &
\textbf{Val. PPL} $\downarrow$ &
\textbf{Avg. Iters} $\downarrow$ &
\textbf{\% Converged} $\uparrow$ \\
\arrayrulecolor{TableRule}\specialrule{1pt}{0pt}{0pt}
Initial only ($\tilde y_0$, no re-injection) & 51.92 & $T_{\max}$ & 12.4\% \\
Concatenation: $\cA_{\theta_a}([\tilde y_t;\tilde y_0])$ & 36.81 & 11.2 & 88.6\% \\
\grouprule
\textcolor{TableAccent}{\textbf{Additive (ours):}} $\cA_{\theta_a}(\tilde y_t + \tilde y_0)$ &
\best{\textbf{34.05}} &
\best{\textbf{8.4}} &
\best{\textbf{99.7\%}} \\
\end{tabularx}
\label{tab:ablation_injection}
\end{table}

\textbf{Backward pass.}
In table~\ref{tab:ablation_backward}, we compare the one-step backward approximation against full implicit differentiation (Anderson on the linear system $(I - J_{\tilde y}^\top) u = v$) and the phantom gradient~\cite{phantomgradient} unroll. Full IFT improves PPL by only 0.14 while increasing peak training memory by $4.8\times$ and step time by $2.7\times$; phantom gradient lies in between. The one-step approximation allows relatively cheap training cost while maintaining nearly all of the original performance of a full backward gradient computation.

\begin{table}[!ht]
\caption{\textbf{Ablation for running the full IFT gradients, phantom gradients, and the one-step approximation.}}
\centering
\small
\setlength{\tabcolsep}{5pt}
\renewcommand{\arraystretch}{1.22}
\color{TableText}
\begin{tabularx}{\linewidth}{@{\hspace{5pt}} l Y Y Y Y @{\hspace{5pt}}}
\arrayrulecolor{TableRule}
\HeaderStrut
\textbf{Backward Pass} &
\textbf{Val. PPL} $\downarrow$ &
\textbf{Train Mem.} $\downarrow$ &
\textbf{Step Time} $\downarrow$ &
\textbf{$\Delta$ vs. Full} \\
\arrayrulecolor{TableRule}\specialrule{1pt}{0pt}{0pt}
Full IFT (Anderson on $u$) & 33.91 & $4.8\times$ & $2.7\times$ & --- \\
Phantom gradient ($k{=}3$)~\cite{phantomgradient} & 34.02 & $1.8\times$ & $1.4\times$ & $+0.11$ \\
\grouprule
\textcolor{TableAccent}{\textbf{Ours:}} one-step ($u\!\approx\!v$)~\cite{jfb} &
\best{\textbf{34.05}} &
\best{$\bm{1.0\times}$} &
\best{$\bm{1.0\times}$} &
$+0.14$ \\
\end{tabularx}
\label{tab:ablation_backward}
\end{table}

\begin{AIbox}{Takeaway: One-step gradient approximation is sufficient for Attractor Models}
Using the one-step gradient approximation enables efficient training, while maintaining nearly all of the quality of using full backward gradients.
\end{AIbox}

\textbf{Initialization ablation.}
Table~\ref{tab:ablation_init} isolates the effect of the solver initialization. Starting the attractor solve from an uninformative state, either zero or Gaussian noise, substantially degrades both quality and convergence. In contrast, initializing from the backbone proposal $\tilde y_0=\mathcal T_{\theta_b}(E(x))$ gives the solver a meaningful output-side starting point, yielding lower perplexity, fewer iterations, and higher downstream accuracy. This supports our hypothesis that the refinement should begin near a coherent prediction embedding rather than constructing one from scratch.

\begin{table}[!ht]
\centering
\caption{\textbf{Effect of attractor solver initialization.} Initializing from the backbone proposal $\tilde y_0$ clearly outperforms uninformative zero or Gaussian noise initialization.}
\label{tab:ablation_init}

\small
\setlength{\tabcolsep}{5pt}
\renewcommand{\arraystretch}{1.18}
\color{TableText}

\begin{tabularx}{\linewidth}{@{\hspace{5pt}} l Y Y Y Y @{\hspace{5pt}}}
\arrayrulecolor{TableRule}
\HeaderStrut
\textbf{Initialization} &
\textbf{Val. PPL} $\downarrow$ &
\textbf{Avg. Iters} $\downarrow$ &
\textbf{\% Converged} $\uparrow$ &
\textbf{Core} $\uparrow$ \\
\arrayrulecolor{TableRule}\specialrule{1pt}{0pt}{0pt}

Zero: $\tilde y_{\mathrm{init}}=0$
& 43.87
& 14.8
& 71.3\%
& $5.42 \pm 0.13$ \\

Gaussian: $\tilde y_{\mathrm{init}}\sim\mathcal N(0,\sigma^2 I)$
& 41.26
& 13.6
& 78.9\%
& $5.71 \pm 0.11$ \\

\grouprule

\textcolor{TableAccent}{\textbf{Backbone proposal:}}
$\tilde y_{\mathrm{init}}=\tilde y_0$
& \best{\textbf{34.05}}
& \best{\textbf{8.4}}
& \best{\textbf{99.7\%}}
& \best{$\bm{6.74 \pm 0.10}$} \\

\end{tabularx}
\end{table}

%% file: sections/conclusion.tex
\vspace{-0.2em}
\section{Conclusion and Future Work}
\label{sec:conclusion}

In this work, we propose Attractor Models, a new family of architectures that first produce meaningful prediction embeddings and then refine them through an attractor module by solving for a fixed point. This formulation makes recurrent refinement stable, memory-efficient, and adaptive, while avoiding the cost of explicit unrolling and achieving strong results across both large-scale language modeling and hard reasoning with tiny models. In stark contrast to prior looped language models, training Attractor Models gives rise to equilibrium internalization: the model learns to make the very refinement procedure that trained it largely unnecessary at inference time. Interesting directions for future work is to further study the equilibrium internalization phenomenon and the differences between Attractor Models and finite-loop recurrence.

\textbf{Acknowledgements.} The authors gratefully acknowledge generous support from Coefficient Giving.

%% file: sections/supplementary.tex
\newpage 
\section{Related Work}
\label{sec:related}

\textbf{Looped and recurrent language models.}
Weight-tied recurrence has re-emerged as an alternative to deeper feed-forward stacks~\cite{weight_tying,inan2017tying,kaiser2016neuralgpuslearnalgorithms,lan2020albertlitebertselfsupervised,dabre2018recurrentstackinglayerscompact,takase2023lessonsparametersharinglayers,bae2025relaxedrecursivetransformerseffective}. Universal Transformers share a single block across depth~\cite{universal_transformer,loop1,sapunov2026universaltransformersneedmemory}, while looped and recurrent language models iterate a recurrent update on a hidden state to enable latent reasoning~\cite{looplm,parcae,geiping,mor,zhengchaingpt,zhangmodr,song2026adaponderlm,knupp2026depth,zeitoun2026hyperloop,mcleish2025teachingpretrainedlanguagemodels,zheng2026chaingpt,zhang2026modr,song2026adaponderlmgatedponderinglanguage,moosa2026understandingdynamiccomputeallocation,knupp2026depthrecurrentattentionmixturesgiving,wu2025parallellooptransformerefficient,zeitoun2026hyperlooptransformers,pappone2025twoscalelatentdynamicsrecurrentdepth,rauba2026tinyautoregressiverecursivemodels,zeng2026ponderlmpretraininglanguagemodels,tur2026recurrentdepthvlaimplicittesttime,komisarczyk2026recursiveinferencemachinesneural,ghugare2026roleiterativecomputationreinforcement,cameron2026stepforwardksteps,williams2026prioritizeprocessjustoutcome,han2026hierarchicalvsflatiteration,goyal2026eltelasticloopedtransformers,tang2026looprptreinforcementpretraininglooped}. Latent (implicit) reasoning methods such as Coconut~\cite{coconut} and related methods~\cite{nextlat,chen2025innerthinkingtransformerleveraging,xu2026formalcomparisonchainthought,jeddi2026loopformerelasticdepthloopedtransformers,fu2026thinkathardselectivelatentiterations,zhou2025coevolutionary,geiping2025efficientparallelsamplersrecurrentdepth,hu2025beliefstatetransformer,fleuret2025freetransformer} perform reasoning in continuous representation space rather than through chain-of-thought tokens. These approaches unroll the loop for a fixed number of steps at training time, causing training memory to grow linearly in depth, coupling inference iterations to training, and often degrading quality when iterations are extended at test time~\cite{looplm}. From theoretical and mechanistic analysis perspectives, looped and continuous thinking models have shown to offer benefits over non-recurrent models \cite{loop2,loop3,loop5,loop6,looped1,giannou2023loopedtransformersprogrammablecomputers,yang2024loopedtransformersbetterlearning,gatmiry2024roledepthloopingincontext,huang2025transformerslearnimplementmultistep,xu2025expressivepowerloopedtransformers,merrill2025littledepthgoeslong,labovich2026stabilitygeneralizationloopedtransformers,zhu2025reasoningsuperpositiontheoreticalperspective,zhu2025emergence,gatmiry2024can,merrill2026little}. Mechanistic analysis shows that the recurrent updates of looped LMs in fact converge to a fixed point for the vast majority of tokens, directly motivating our formulation~\cite{mech}.

\textbf{Implicit fixed-point models.}
Deep Equilibrium Models (DEQs) replace finite unrolling with a fixed-point equation $z^\star = f_\theta(z^\star, x)$ and train through it via implicit differentiation, decoupling effective depth from training memory~\cite{deq}. Solvers typically use Anderson acceleration~\cite{anderson}, with backward passes ranging from full implicit differentiation to cheap surrogates such as one-step~\cite{jfb} and phantom~\cite{phantomgradient} gradients. Standard DEQs solve for a hidden state initialized from zero and decode through a separate output head. In Attractor Models, the equilibrium instead lives in the tied output-embedding space and is warm-started from a meaningful backbone prediction; we find this structure essential for stable training at language-modeling scale and show it gives rise to equilibrium internalization (Section~\ref{sec:exp-test}).

\textbf{Tiny recursive reasoners.}
On small-data algorithmic benchmarks such as Sudoku and maze solving, hierarchical and tiny recursive networks (HRM, TRM)~\cite{hrm,trm} achieve strong accuracy with few parameters, but exhibit a ``less is more'' behavior where performance collapses as model size grows~\cite{trm}. Attractor Models retain the iterative-refinement benefit of these architectures while scaling cleanly with parameter count).

\textbf{Positioning.}
Looped LMs couple three quantities through unrolled BPTT: inference depth, training depth, and training memory~\cite{looplm}. Attractor Models decouple all three: the equilibrium is defined by a residual equation in output-embedding space, the number of solver evaluations is chosen adaptively by tolerance $\varepsilon$, and training memory in the recurrent block is constant in effective depth. Empirically, this combines the strengths of feed-forward and recurrent models—better perplexity than parameter-matched Transformers, 25--31\% lower training FLOPs and constant memory relative to looped LMs, and clean scaling on hard reasoning where specialized recursive networks collapse.

\section{Theory of Looped and Attractor Models}
\label{sec:theory}

\subsection{Well-posedness and the implicit gradient}

Fix an input $x$ and parameters $\theta$. Throughout this section, we write $F(y) \triangleq f_\theta(y, E(x))$ for brevity, so that the fixed-point equation $y^\star = f_\theta(y^\star, E(x))$ becomes $y^\star = F(y^\star)$. Let $J_F(y) \triangleq \partial F / \partial y$ denote the Jacobian of $F$ with respect to its state input. We work in a generic norm $\|\cdot\|$ on $\mathbb{R}^{L \times d}$ and its induced operator norm; for concreteness, can take both to be Frobenius and spectral, respectively. Assumptions are borrowed from~\cite{mdeq,krantz}.

\begin{assumption}[Local contraction]
\label{ass:contraction}
There exist a point $\bar y \in \mathbb{R}^{L \times d}$, a radius
$r > 0$, and a constant $L \in [0, 1)$ such that
\begin{enumerate}
    \item $F$ maps the closed ball
          $B_r(\bar y) \triangleq \{y : \|y - \bar y\| \le r\}$
          into itself, and
    \item $F$ is $L$-Lipschitz on $B_r(\bar y)$:
          \[
              \|F(y) - F(y')\| \;\le\; L \,\|y - y'\|
              \qquad \text{for all } y, y' \in B_r(\bar y).
          \]
\end{enumerate}
\end{assumption}

When $F$ is continuously differentiable, a sufficient form of (ii) is $\sup_{y \in B_r(\bar y)} \|J_F(y)\| \le L$.

\begin{theorem}[Well-posedness and implicit gradient]
\label{thm:wellposed}
Under Assumption~\ref{ass:contraction}, the following hold.
\begin{enumerate}
    \item \emph{Existence and uniqueness:} There is a unique $y^\star \in B_r(\bar y)$ with $y^\star = F(y^\star)$.
    \item \emph{The picard iteration converges linearly:} For any $y_0 \in B_r(\bar y)$, the iterates $y_{k+1}= F(y_k)$ remain in $B_r(\bar y)$ and satisfy
    \begin{equation}
        \|y_k - y^\star\| \;\le\; L^k \,\|y_0 - y^\star\|.
        \label{eq:picard_rate}
    \end{equation}
    \item \emph{Validity of the implicit gradient:} If $F$ is continuously differentiable on $B_r(\bar y)$, then $I - J_F(y^\star)$ is invertible, $y^\star$ depends continuously differentiably on $\theta$ in a neighborhood of the current parameters, and
    \begin{equation}
        \frac{\partial y^\star}{\partial \theta}
        \;=\;
        \bigl(I - J_F(y^\star)\bigr)^{-1}
        \left.\frac{\partial f_\theta}{\partial \theta}\right|_{y^\star}.
        \label{eq:implicit_grad_thm}
    \end{equation}
\end{enumerate}
\end{theorem}

\begin{proof}
Parts (i) and (ii) follow from the Banach fixed-point theorem applied to the contraction $F$ on the closed (hence complete) ball $B_r(\bar y) \subset (\mathbb{R}^{L \times d}, \|\cdot\|)$.

For (iii), Lipschitz continuity with constant $L$ gives $\|J_F(y^\star)\| \le L < 1$, so the spectral radius satisfies $\rho(J_F(y^\star)) \le \|J_F(y^\star)\| < 1$. The Neumann series $\sum_{k \ge 0} J_F(y^\star)^k$ therefore converges to $(I - J_F(y^\star))^{-1}$, and in particular $I - J_F(y^\star)$ is invertible. Define $G(y, \theta) \triangleq f_\theta(y, E(x)) - y$, which is $C^1$ in both arguments and satisfies $G(y^\star, \theta) = 0$ with

\begin{equation}
    \left.\frac{\partial G}{\partial y}\right|_{y^\star}
    \;=\; J_F(y^\star) - I,
\end{equation}

which is invertible. The implicit function theorem applied to $G = 0$ yields a unique $C^1$ map $\theta \mapsto y^\star(\theta)$ in a neighborhood of the current parameters, with

\begin{equation}
    \frac{\partial y^\star}{\partial \theta}
    \;=\;
    -\left(\frac{\partial G}{\partial y}\right)^{-1}
     \frac{\partial G}{\partial \theta}
    \;=\;
    \bigl(I - J_F(y^\star)\bigr)^{-1}
    \left.\frac{\partial f_\theta}{\partial \theta}\right|_{y^\star}.
\end{equation}

\end{proof}

Equation~\eqref{eq:implicit_grad_thm} is exactly the gradient computed in the backward pass of Section~\ref{sec:method}. The linear solve is convergent for the same reason: since $\|J_F^\top(y^\star)\| = \|J_F(y^\star)\| \le L < 1$, the Neumann iteration $u \leftarrow J_F^\top(y^\star)\, u + v$ converges geometrically with rate $L$ to the unique solution $u = (I - J_F^\top(y^\star))^{-1} v$.

\subsection{Looped language models are fixed-point iterators}
A LoopLM~\cite{looplm} of depth $T$ with the same weight-tied block $f_\theta$ and warm start $y_0 = E(x)$ produces
\begin{equation}
    y_T^{\mathrm{loop}}
    \;\triangleq\;
    F^{(T)}\bigl(E(x)\bigr)
    \;=\;
    \underbrace{F \circ F \circ \cdots \circ F}_{T \text{ times}}\bigl(E(x)\bigr).
    \label{eq:looplm_unroll}
\end{equation}

This is exactly $T$ Picard iterations of our residual $g_\theta(\cdot, x)$ warm-started from the input embedding. See the following corollary:

\begin{corollary}[LoopLM as a truncated approximation]
\label{cor:looplm}
Suppose Assumption~\ref{ass:contraction} holds and the warm start $E(x) \in B_r(\bar y)$. Then for every $T\ge 0$,
\begin{equation}
    \|y_T^{\mathrm{loop}} - y^\star\|
    \;\le\;
    L^T \,\|E(x) - y^\star\|.
    \label{eq:looplm_bound}
\end{equation}
In particular, $y_T^{\mathrm{loop}} \to y^\star$ as $T \to \infty$, and the discrepancy between a LoopLM of depth $T$ and our model decays geometrically in $T$.
\end{corollary}

\begin{proof}
Apply Theorem~\ref{thm:wellposed}(ii) with $y_0 = E(x)$.
\end{proof}

Two consequences are worth highlighting:

\textbf{The limit:}
Our fixed-point model is the $T \to \infty$ limit of LoopLM with shared parameters. Training a LoopLM at any finite depth $T$ implicitly trains an approximation to the same equilibrium $y^\star$, with the approximation error controlled by~\eqref{eq:looplm_bound}.

\textbf{The iteration count at inference:}
Picard iteration converges at the linear rate $L$. Our forward pass uses Anderson acceleration, which under standard regularity conditions converges \emph{superlinearly} near $y^\star$~\cite{anderson}. To reach a target residual tolerance $\varepsilon$, the root finder therefore typically requires fewer evaluations of $f_\theta$ than the unroll depth $T$ a LoopLM would need for a comparable residual. The exact crossover depends on $L$ and on the solver hyperparameters.

\textbf{Caveat:}
Assumption~\ref{ass:contraction} is local and is not guaranteed to hold for an arbitrarily trained transformer block. In our experiments, we generally notice that $\|J_F(y^\star)\| \le L < 1$, the authors of DEQ also find that using a regularizer on the Jacobian to enforce this helps~\cite{jacreg}.

\subsection{Where looped language models fall short}
\label{subsec:fall}
\subsubsection{Attractor Models}

The initialization $\tilde y_0$ and the fixed point $\tilde y^\star$ live in the same tied output-embedding space, so either can be decoded by the unembedding $E^\top$. The loss depends only on $\tilde y^\star$, and the implicit gradient with respect to the backbone parameters is
\[
    \nabla_{\theta_b}\mathcal L
    \;=\;
    u^\top J_{\tilde y_0}\,\nabla_{\theta_b}\tilde y_0,
    \qquad
    u=(I-J_{\tilde y}^\top)^{-1}v .
\]
Thus, the loss may be viewed as a function of the backbone's output embedding: perturbations in $\tilde y_0$ affect $\mathcal L$ by perturbing $\tilde y^\star$ in the same space. The backbone is therefore optimized as a standalone next-token predictor whose embedding is decoded by $E^\top$.

When the backbone has strictly greater capacity than the weight-tied attractor, the joint optimum places the prediction work in the backbone. Consequently, $\tilde y_0(x;\theta_b^\star)$ approaches the loss-minimizing embedding. The fixed-point constraint
\[
    \cT_{\theta_a^\star}(\tilde y^\star,\tilde y_0)=\tilde y^\star
\]
is then consistent only when $\tilde y^\star=\tilde y_0$, yielding
\[
    \cT_{\theta_a^\star}(\tilde y_0,\tilde y_0)=\tilde y_0 .
\]

\subsubsection{Standard looped LMs}

In standard looped language models, the latent initialization $h_0$ is uninformative, typically zero or noise, and occupies the role of a hidden state rather than an output embedding. Decoding $h_0E^\top$ is therefore not meaningful.

Gradients instead flow by backpropagation through time along the unrolled trajectory $(h_0,\dots,h_T)$. There is no term in this gradient that drives $h_0$ toward the loss-bearing final state $h_T$ in embedding distance: the initialization is fixed by the architecture and is not produced by a separately trained predictor. \qed

\subsubsection{Implicit-Gradient Barrier}

The implicit gradient
\[
    \nabla_\theta \mathcal{L}_\infty
    =
    v^\top (I-J_g^\top)^{-1}\partial_\theta g
\]
contains the factor $(I-J_g)^{-1}$. This inverse is undefined when $1\in\sigma(J_g)$, and its operator norm scales as
\[
    \|(I-J_g)^{-1}\|
    \sim
    \mathrm{dist}(1,\sigma(J_g))^{-1}
\]
near such a singularity.

Along any continuous gradient-descent path, the eigenvalues of $J_g(y^\star;\theta)$ vary continuously. Therefore, for the dominant eigenvalue to leave the unit disk through $+1$---the canonical loss-of-contraction route for residual iteration maps---it must approach $1$. This forces $\|(I-J_g)^{-1}\|\to\infty$.

Unless $v$ is orthogonal to the offending eigendirection, a codimension-one and hence non-generic condition, the gradient norm diverges. The descent step therefore blows up before the boundary can be crossed, confining $\theta$ to the contractive region $\{\rho(J_g)<1\}$.

For exits through complex eigenvalues, the same conclusion holds under the one-step JFB approximation, since the truncated Neumann series
\[
    \sum_k (J_g^\top)^k
\]
diverges as $\rho(J_g)\to 1$. \qed

\subsubsection{Interpretation}

Fixed-loop training has no inherent mechanism that favors contractive iterations. An optimum in which $g_\theta$ perturbs the embedding for $K$ steps and only happens to land accurately at step $K$ is a valid fixed-loop solution. These are precisely the solutions that fail when extra loops are run at inference.

Equilibrium training cannot reach such solutions, because the implicit gradient creates a barrier of diverging gradients around non-contractive regions. As a result, the trajectory of $\theta$ is confined to the regime in which the solver converges, the one-step gradient is a descent direction, and additional iterations remain stable.

The mechanistic observation~\cite{mech} that fixed points emerge only late in standard looped training is consistent with this picture: the loss landscape contains a basin near contractive solutions, but only equilibrium training applies pressure toward that basin from the beginning of training.

\newpage 

\section{Hyperparamter and Experimental Settings}
\label{sec:hyperparam}
We use NVIDIA H200 GPUs for all of our experiments.

\begin{table}[!ht]
\centering
\small
\setlength{\tabcolsep}{4pt}
\renewcommand{\arraystretch}{1.22}
\caption{Architecture hyperparameters for each model family and scale.}
\label{tab:arch_hyperparams}
\begin{tabular}{l ccc ccc ccc}
\toprule
& \multicolumn{3}{c}{\textbf{Transformer}} & \multicolumn{3}{c}{\textbf{Parcae}} & \multicolumn{3}{c}{\textbf{Attractor Model}} \\
\cmidrule(lr){2-4} \cmidrule(lr){5-7} \cmidrule(lr){8-10}
\textbf{Hyperparameter} & 140M & 370M & 770M & 140M & 370M & 770M & 140M & 370M & 770M \\
\midrule
$d_\text{model}$        & 768  & 1024 & 1280 & 768  & 1024 & 1280 & 1024 & 1280 & 1280 \\
$d_\text{ff}$            & 3072 & 4096 & 5120 & 3072 & 4096 & 5120 & 4096 & 5120 & 5120 \\
Attention heads          & 6    & 8    & 10   & 6    & 8    & 10   & 8    & 10   & 10   \\
\midrule
Total layers             & 6    & 12   & 18   & 6    & 12   & 18   & 8    & 17   & 37   \\
\quad Prelude            & ---  & ---  & ---  & 2    & 4    & 6    & 7    & 15   & 35   \\
\quad Core (recurrent)   & ---  & ---  & ---  & 2    & 4    & 6    & 1    & 2    & 2    \\
\quad Coda               & ---  & ---  & ---  & 2    & 4    & 6    & 0    & 0    & 0    \\
Mean recurrence $T$      & 1    & 1    & 1    & 8    & 8    & 8    & \multicolumn{3}{c}{(solver-determined)} \\
\midrule
Sequence length          & \multicolumn{9}{c}{2048} \\
Vocab size               & \multicolumn{9}{c}{32768} \\
Tie embeddings           & \multicolumn{9}{c}{\cmark} \\
Norm                     & \multicolumn{9}{c}{RMSNorm ($\epsilon = 10^{-5}$)} \\
Activation               & \multicolumn{9}{c}{ReLU$^2$} \\
QK-norm                  & \multicolumn{9}{c}{\cmark} \\
Precision                & \multicolumn{9}{c}{bf16-mixed} \\
\bottomrule
\end{tabular}
\end{table}

\begin{table}[!ht]
\centering
\small
\setlength{\tabcolsep}{5pt}
\renewcommand{\arraystretch}{1.22}
\caption{Fixed-point solver hyperparameters for Attractor Model. Parcae and Transformer do not use a solver.}
\label{tab:solver_hyperparams}
\begin{tabular}{l c c c}
\toprule
\textbf{Hyperparameter} & \textbf{140M} & \textbf{370M} & \textbf{770M} \\
\midrule
Forward solver & Anderson & Anderson & Anderson \\
Max iterations (fwd) & 64 & 64 & 32 \\
Min iterations (fwd) & 6 & 8 & 6 \\
Tolerance (fwd) & $3 \times 10^{-4}$ & $2 \times 10^{-4}$ & $1.5 \times 10^{-4}$ \\
Anderson $m$ & 5 & 5 & 5 \\
Anderson $\beta$ & 1.0 & 1.0 & 0.96 \\
\midrule
Backward type & one-step IFT & one-step IFT & Anderson \\
Max iterations (bwd) & 64 & 64 & 32 \\
Min iterations (bwd) & 6 & 6 & 6 \\
Tolerance (bwd) & $3 \times 10^{-4}$ & $2 \times 10^{-4}$ & $1.5 \times 10^{-4}$ \\
Adjoint grad clip & --- & --- & 2.0 \\
\midrule
Layer-scale init ($\gamma_0$) & 0.75 & 0.75 & 0.73 \\
$\gamma_{\max}$ & 0.75 & 1.0 & 0.9 \\
FP LR scale & 0.5 & 0.4 & 0.4 \\
FP weight decay & 0.1 & 0.1 & 0.1 \\
\bottomrule
\end{tabular}
\end{table}

\begin{table}[!ht]
\centering
\small
\setlength{\tabcolsep}{4pt}
\renewcommand{\arraystretch}{1.22}
\caption{Training hyperparameters. All three model families share the same optimizer, schedule, and data configuration at each scale. Differences are noted where they occur.}
\label{tab:train_hyperparams}
\begin{tabular}{l ccc}
\toprule
\textbf{Hyperparameter} & \textbf{140M} & \textbf{370M} & \textbf{770M} \\
\midrule
Training tokens & 11.2B & 29.6B & 61.6B \\
Dataset & \multicolumn{3}{c}{FineWeb-Edu 100B (shuffled)} \\
Tokenizer & \multicolumn{3}{c}{BPE, 32768 vocab} \\
\midrule
Optimizer & \multicolumn{3}{c}{MuonAdamW} \\
AdamW LR & $8 \times 10^{-3}$ & $8 \times 10^{-3}$ & $5$--$6 \times 10^{-3}$\textsuperscript{$\dagger$} \\
AdamW $(\beta_1, \beta_2)$ & \multicolumn{3}{c}{(0.8, 0.95)} \\
AdamW $\epsilon$ & \multicolumn{3}{c}{$10^{-10}$} \\
AdamW weight decay & \multicolumn{3}{c}{0} \\
Muon LR & $8 \times 10^{-3}$ & $8 \times 10^{-3}$ & $6$--$8 \times 10^{-3}$\textsuperscript{$\dagger$} \\
Muon momentum & \multicolumn{3}{c}{0.95} \\
Muon weight decay & \multicolumn{3}{c}{0.2 (linear decay to 0)} \\
Muon NS steps & \multicolumn{3}{c}{5} \\
\midrule
LR schedule & \multicolumn{3}{c}{Trapezoid (0\% warmup, 50\% cooldown)} \\
Min LR & \multicolumn{3}{c}{0} \\
Gradient clipping & \multicolumn{3}{c}{1.0} \\
\midrule
Global batch size (tokens) & \multicolumn{3}{c}{524K ($= 256 \times 2048$)} \\
Micro batch size & 32 & 32 & 8--32\textsuperscript{$\dagger$} \\
Gradient checkpointing & \xmark & \xmark & model-dependent\textsuperscript{$\ddagger$} \\
\texttt{torch.compile} & \cmark\textsuperscript{$\star$} & \cmark\textsuperscript{$\star$} & \cmark\textsuperscript{$\star$} \\
Strategy & \multicolumn{3}{c}{DDP} \\
\midrule
Init strategy & \multicolumn{3}{c}{scaled-zero + orthogonal} \\
Seed & \multicolumn{3}{c}{42} \\
\bottomrule
\end{tabular}

\vspace{0.5em}
\raggedright
\footnotesize
\textsuperscript{$\dagger$}At 770M, Attractor Model uses AdamW LR $5 \times 10^{-3}$ and Muon LR $6 \times 10^{-3}$; Transformer/Parcae use $6 \times 10^{-3}$ / $8 \times 10^{-3}$.\\
\textsuperscript{$\ddagger$}Gradient checkpointing enabled for 770M Attractor Model and Parcae; disabled for Transformer.\\
\textsuperscript{$\star$}\texttt{torch.compile} enabled for Transformer and Parcae; disabled for Attractor Model (implicit-gradient hooks are not compile-compatible).
\end{table}

\begin{table}[!ht]
\centering
\small
\setlength{\tabcolsep}{5pt}
\renewcommand{\arraystretch}{1.22}
\caption{Parcae recurrence hyperparameters.}
\label{tab:parcae_recurrence}
\begin{tabular}{l c c c}
\toprule
\textbf{Hyperparameter} & \textbf{140M} & \textbf{370M} & \textbf{770M} \\
\midrule
Injection type & \multicolumn{3}{c}{diagonal} \\
State init & \multicolumn{3}{c}{like-init} \\
Mean recurrence & 8 & 8 & 8 \\
Mean backprop depth & 4 & 4 & 4 \\
Sampling scheme & \multicolumn{3}{c}{poisson-truncated-full} \\
Iteration method & \multicolumn{3}{c}{per-sequence} \\
Prelude norm & \multicolumn{3}{c}{\cmark} \\
\bottomrule
\end{tabular}
\end{table}

%% file: main.bib
@inproceedings{deq,
  title     = {Deep Equilibrium Models},
  author    = {Bai, Shaojie and Kolter, J. Zico and Koltun, Vladlen},
  booktitle = {Advances in Neural Information Processing Systems 32 (NeurIPS 2019)},
  year      = {2019}
}

@inproceedings{weight_tying,
  author    = {Press, Ofir and Wolf, Lior},
  title     = {Using the Output Embedding to Improve Language Models},
  booktitle = {Proceedings of the 15th Conference of the European Chapter of the Association for Computational Linguistics: Volume 2, Short Papers},
  pages     = {157--163},
  year      = {2017}
}

@misc{looplm,
      title={Scaling Latent Reasoning via Looped Language Models}, 
      author={Rui-Jie Zhu and Zixuan Wang and Kai Hua and Tianyu Zhang and Ziniu Li and Haoran Que and Boyi Wei and Zixin Wen and Fan Yin and He Xing and Lu Li and Jiajun Shi and Kaijing Ma and Shanda Li and Taylor Kergan and Andrew Smith and Xingwei Qu and Mude Hui and Bohong Wu and Qiyang Min and Hongzhi Huang and Xun Zhou and Wei Ye and Jiaheng Liu and Jian Yang and Yunfeng Shi and Chenghua Lin and Enduo Zhao and Tianle Cai and Ge Zhang and Wenhao Huang and Yoshua Bengio and Jason Eshraghian},
      year={2025},
      eprint={2510.25741},
      archivePrefix={arXiv},
      primaryClass={cs.CL},
      url={https://arxiv.org/abs/2510.25741}, 
}

@misc{parcae,
      title={Parcae: Scaling Laws For Stable Looped Language Models}, 
      author={Hayden Prairie and Zachary Novack and Taylor Berg-Kirkpatrick and Daniel Y. Fu},
      year={2026},
      eprint={2604.12946},
      archivePrefix={arXiv},
      primaryClass={cs.LG},
      url={https://arxiv.org/abs/2604.12946}, 
}

@article{merrill2026little,
  title={A little depth goes a long way: {T}he expressive power of log-depth transformers},
  author={Merrill, Will and Sabharwal, Ashish},
  journal={Advances in Neural Information Processing Systems},
  volume={38},
  pages={95315--95339},
  year={2026}
}

@inproceedings{transformer,
  title     = {Attention is All You Need},
  author    = {Vaswani, Ashish and Shazeer, Noam and Parmar, Niki and
               Uszkoreit, Jakob and Jones, Llion and Gomez, Aidan N. and
               Kaiser, Lukasz and Polosukhin, Illia},
  booktitle = {Advances in Neural Information Processing Systems 30},
  pages     = {5998--6008},
  year      = {2017}
}

@misc{universal_transformer,
      title={Universal Transformers}, 
      author={Mostafa Dehghani and Stephan Gouws and Oriol Vinyals and Jakob Uszkoreit and Łukasz Kaiser},
      year={2019},
      eprint={1807.03819},
      archivePrefix={arXiv},
      primaryClass={cs.CL},
      url={https://arxiv.org/abs/1807.03819}, 
}

@misc{coconut,
      title={Training Large Language Models to Reason in a Continuous Latent Space}, 
      author={Shibo Hao and Sainbayar Sukhbaatar and DiJia Su and Xian Li and Zhiting Hu and Jason Weston and Yuandong Tian},
      year={2025},
      eprint={2412.06769},
      archivePrefix={arXiv},
      primaryClass={cs.CL},
      url={https://arxiv.org/abs/2412.06769}, 
}

@misc{inan2017tying,
      title={Tying Word Vectors and Word Classifiers: A Loss Framework for Language Modeling}, 
      author={Hakan Inan and Khashayar Khosravi and Richard Socher},
      year={2017},
      eprint={1611.01462},
      archivePrefix={arXiv},
      primaryClass={cs.LG},
      url={https://arxiv.org/abs/1611.01462}, 
}

@article{anderson,
  title   = {Iterative Procedures for Nonlinear Integral Equations},
  author  = {Anderson, Donald G.},
  journal = {Journal of the ACM},
  volume  = {12},
  number  = {4},
  pages   = {547--560},
  year    = {1965},
  doi     = {10.1145/321296.321305}
}

@misc{mech,
      title={A Mechanistic Analysis of Looped Reasoning Language Models}, 
      author={Hugh Blayney and Álvaro Arroyo and Johan Obando-Ceron and Pablo Samuel Castro and Aaron Courville and Michael M. Bronstein and Xiaowen Dong},
      year={2026},
      eprint={2604.11791},
      archivePrefix={arXiv},
      primaryClass={cs.LG},
      url={https://arxiv.org/abs/2604.11791}, 
}

@misc{jacreg,
      title={Stabilizing Equilibrium Models by Jacobian Regularization}, 
      author={Shaojie Bai and Vladlen Koltun and J. Zico Kolter},
      year={2021},
      eprint={2106.14342},
      archivePrefix={arXiv},
      primaryClass={cs.LG},
      url={https://arxiv.org/abs/2106.14342}, 
}

@inproceedings{mdeq,
  author    = {Bai, Shaojie and Koltun, Vladlen and Kolter, J. Zico},
  title     = {Multiscale Deep Equilibrium Models},
  booktitle = {Advances in Neural Information Processing Systems},
  volume    = {33},
  year      = {2020},
  publisher = {Curran Associates, Inc.},
  url       = {https://proceedings.neurips.cc/paper/2020/hash/3812f9a59b634c2a9c574610eaba5bed-Abstract.html}
}

@book{krantz,
  author    = {Krantz, Steven G. and Parks, Harold R.},
  title     = {The Implicit Function Theorem: History, Theory, and Applications},
  publisher = {Birkh{\"a}user},
  address   = {Boston, MA},
  year      = {2002},
  doi       = {10.1007/978-1-4612-0059-8},
  isbn      = {978-0-8176-4285-3}
}

@misc{hrm,
      title={Hierarchical Reasoning Model}, 
      author={Guan Wang and Jin Li and Yuhao Sun and Xing Chen and Changling Liu and Yue Wu and Meng Lu and Sen Song and Yasin Abbasi Yadkori},
      year={2025},
      eprint={2506.21734},
      archivePrefix={arXiv},
      primaryClass={cs.AI},
      url={https://arxiv.org/abs/2506.21734}, 
}

@misc{trm,
      title={Less is More: Recursive Reasoning with Tiny Networks}, 
      author={Alexia Jolicoeur-Martineau},
      year={2025},
      eprint={2510.04871},
      archivePrefix={arXiv},
      primaryClass={cs.LG},
      url={https://arxiv.org/abs/2510.04871}, 
}

@misc{act,
      title={Adaptive Computation Time for Recurrent Neural Networks}, 
      author={Alex Graves},
      year={2017},
      eprint={1603.08983},
      archivePrefix={arXiv},
      primaryClass={cs.NE},
      url={https://arxiv.org/abs/1603.08983}, 
}

@inproceedings{giannou2023looped,
  title={Looped transformers as programmable computers},
  author={Giannou, Angeliki and Rajput, Shashank and Sohn, Jy-yong and Lee, Kangwook and Lee, Jason D and Papailiopoulos, Dimitris},
  booktitle={International Conference on Machine Learning},
  pages={11398--11442},
  year={2023},
  organization={PMLR}
}

@article{ozeren2025reinforcement,
  title={Reinforcement Learning for Latent-Space Thinking in {LLM}s},
  author={{\"O}zeren, Enes and A{\ss}enmacher, Matthias},
  journal={arXiv preprint arXiv:2512.11816},
  year={2025}
}

@misc{geiping,
      title={Scaling up Test-Time Compute with Latent Reasoning: A Recurrent Depth Approach}, 
      author={Jonas Geiping and Sean McLeish and Neel Jain and John Kirchenbauer and Siddharth Singh and Brian R. Bartoldson and Bhavya Kailkhura and Abhinav Bhatele and Tom Goldstein},
      year={2025},
      eprint={2502.05171},
      archivePrefix={arXiv},
      primaryClass={cs.LG},
      url={https://arxiv.org/abs/2502.05171}, 
}

@article{rizvi2026illusion,
  title={The Illusion of Superposition? {A} Principled Analysis of Latent Thinking in Language Models},
  author={Rizvi-Martel, Michael and Rabusseau, Guillaume and Mosbach, Marius},
  journal={arXiv preprint arXiv:2604.06374},
  year={2026}
}

@article{deng2025latent,
  title={Latent reasoning in {LLMs} as a vocabulary-space superposition},
  author={Deng, Jingcheng and Pang, Liang and Wei, Zihao and Xu, Shichen and Duan, Zenghao and Xu, Kun and Song, Yang and Shen, Huawei and Cheng, Xueqi},
  journal={arXiv preprint arXiv:2510.15522},
  year={2025}
}

@article{deng2026latent,
  title={{Latent-GRPO: G}roup Relative Policy Optimization for Latent Reasoning},
  author={Deng, Jingcheng and Wei, Zihao and Pang, Liang and Wu, Junhong and Xu, Shicheng and Duan, Zenghao and Shen, Huawei},
  journal={arXiv preprint arXiv:2604.27998},
  year={2026}
}

@misc{phantomgradient,
      title={On Training Implicit Models}, 
      author={Zhengyang Geng and Xin-Yu Zhang and Shaojie Bai and Yisen Wang and Zhouchen Lin},
      year={2022},
      eprint={2111.05177},
      archivePrefix={arXiv},
      primaryClass={cs.LG},
      url={https://arxiv.org/abs/2111.05177}, 
}

@article{zhou2025coevolutionary,
  title={Coevolutionary continuous discrete diffusion: Make your diffusion language model a latent reasoner},
  author={Zhou, Cai and Yang, Chenxiao and Hu, Yi and Wang, Chenyu and Zhang, Chubin and Zhang, Muhan and Mackey, Lester and Jaakkola, Tommi and Bates, Stephen and Zhang, Dinghuai},
  journal={arXiv preprint arXiv:2510.03206},
  year={2025}
}

@inproceedings{mcleish2025teachingpretrainedlanguagemodels,
  title={Teaching Pretrained Language Models to Think Deeper with Retrofitted Recurrence},
  author={McLeish, Sean Michael and Li, Ang and Kirchenbauer, John and Kalra, Dayal Singh and Bartoldson, Brian R and Kailkhura, Bhavya and Schwarzschild, Avi and Geiping, Jonas and Goldblum, Micah and Goldstein, Tom},
  booktitle={NeurIPS 2025 Workshop on Efficient Reasoning}
}

@misc{jfb,
      title={JFB: Jacobian-Free Backpropagation for Implicit Networks}, 
      author={Samy Wu Fung and Howard Heaton and Qiuwei Li and Daniel McKenzie and Stanley Osher and Wotao Yin},
      year={2021},
      eprint={2103.12803},
      archivePrefix={arXiv},
      primaryClass={cs.LG},
      url={https://arxiv.org/abs/2103.12803}, 
}

@article{achiam2023gpt,
  title={{GPT}-4 technical report},
  author={OpenAI},
  journal={arXiv preprint arXiv:2303.08774},
  year={2023}
}

@misc{anthropic2024claude3,
  title        = {The Claude 3 Model Family: {Opus, Sonnet, H}aiku},
  author       = {{Anthropic}},
  year         = {2024},
  howpublished = {Model card},
  url          = {https://www-cdn.anthropic.com/de8ba9b01c9ab7cbabf5c33b80b7bbc618857627/Model_Card_Claude_3.pdf}
}

@inproceedings{grattafiori2024llama,
  title={The {L}lama 3 herd of models},
  author={Grattafiori, Aaron and Dubey, Abhimanyu and Jauhri, Abhinav and Pandey, Abhinav and Kadian, Abhishek and Al-Dahle, Ahmad and Letman, Aiesha and Mathur, Akhil and Schelten, Alan and Vaughan, Alex and others},
  booktitle={Neural Information Processing Systems},
  year={2024}
}

@article{team2023gemini,
  title={Gemini: {A} family of highly capable multimodal models},
  author={{Gemini Team}},
  journal={arXiv preprint arXiv:2312.11805},
  year={2023}
}

@misc{nanochat,
  author = {Andrej Karpathy},
  title = {nanochat: The best ChatGPT that \$100 can buy},
  year = {2025},
  publisher = {GitHub},
  url = {https://github.com/karpathy/nanochat}
}

@misc{fineweb,
      title={The FineWeb Datasets: Decanting the Web for the Finest Text Data at Scale}, 
      author={Guilherme Penedo and Hynek Kydlíček and Loubna Ben allal and Anton Lozhkov and Margaret Mitchell and Colin Raffel and Leandro Von Werra and Thomas Wolf},
      year={2024},
      eprint={2406.17557},
      archivePrefix={arXiv},
      primaryClass={cs.CL},
      url={https://arxiv.org/abs/2406.17557}, 
}

@inproceedings{lambada,
  title = "The {LAMBADA} dataset: Word prediction requiring a broad discourse context",
  author = "Paperno, Denis and
    Kruszewski, Germ{\'a}n and
    Lazaridou, Angeliki and
    Pham, Ngoc Quan and
    Bernardi, Raffaella and
    Pezzelle, Sandro and
    Baroni, Marco and
    Boleda, Gemma and
    Fern{\'a}ndez, Raquel",
  booktitle = "Proceedings of the 54th Annual Meeting of the Association for Computational Linguistics (Volume 1: Long Papers)",
  month = aug,
  year = "2016",
  address = "Berlin, Germany",
  publisher = "Association for Computational Linguistics",
  pages = "1525--1534",
  doi = "10.18653/v1/P16-1144",
  url = "https://aclanthology.org/P16-1144/"
}

@article{r1,
   title={DeepSeek-R1 incentivizes reasoning in LLMs through reinforcement learning},
   volume={645},
   ISSN={1476-4687},
   url={http://dx.doi.org/10.1038/s41586-025-09422-z},
   DOI={10.1038/s41586-025-09422-z},
   number={8081},
   journal={Nature},
   publisher={Springer Science and Business Media LLC},
   author={Guo, Daya and Yang, Dejian and Zhang, Haowei and Song, Junxiao and Wang, Peiyi and Zhu, Qihao and Xu, Runxin and Zhang, Ruoyu and Ma, Shirong and Bi, Xiao and Zhang, Xiaokang and Yu, Xingkai and Wu, Yu and Wu, Z. F. and Gou, Zhibin and Shao, Zhihong and Li, Zhuoshu and Gao, Ziyi and Liu, Aixin and Xue, Bing and Wang, Bingxuan and Wu, Bochao and Feng, Bei and Lu, Chengda and Zhao, Chenggang and Deng, Chengqi and Ruan, Chong and Dai, Damai and Chen, Deli and Ji, Dongjie and Li, Erhang and Lin, Fangyun and Dai, Fucong and Luo, Fuli and Hao, Guangbo and Chen, Guanting and Li, Guowei and Zhang, H. and Xu, Hanwei and Ding, Honghui and Gao, Huazuo and Qu, Hui and Li, Hui and Guo, Jianzhong and Li, Jiashi and Chen, Jingchang and Yuan, Jingyang and Tu, Jinhao and Qiu, Junjie and Li, Junlong and Cai, J. L. and Ni, Jiaqi and Liang, Jian and Chen, Jin and Dong, Kai and Hu, Kai and You, Kaichao and Gao, Kaige and Guan, Kang and Huang, Kexin and Yu, Kuai and Wang, Lean and Zhang, Lecong and Zhao, Liang and Wang, Litong and Zhang, Liyue and Xu, Lei and Xia, Leyi and Zhang, Mingchuan and Zhang, Minghua and Tang, Minghui and Zhou, Mingxu and Li, Meng and Wang, Miaojun and Li, Mingming and Tian, Ning and Huang, Panpan and Zhang, Peng and Wang, Qiancheng and Chen, Qinyu and Du, Qiushi and Ge, Ruiqi and Zhang, Ruisong and Pan, Ruizhe and Wang, Runji and Chen, R. J. and Jin, R. L. and Chen, Ruyi and Lu, Shanghao and Zhou, Shangyan and Chen, Shanhuang and Ye, Shengfeng and Wang, Shiyu and Yu, Shuiping and Zhou, Shunfeng and Pan, Shuting and Li, S. S. and Zhou, Shuang and Wu, Shaoqing and Yun, Tao and Pei, Tian and Sun, Tianyu and Wang, T. and Zeng, Wangding and Liu, Wen and Liang, Wenfeng and Gao, Wenjun and Yu, Wenqin and Zhang, Wentao and Xiao, W. L. and An, Wei and Liu, Xiaodong and Wang, Xiaohan and Chen, Xiaokang and Nie, Xiaotao and Cheng, Xin and Liu, Xin and Xie, Xin and Liu, Xingchao and Yang, Xinyu and Li, Xinyuan and Su, Xuecheng and Lin, Xuheng and Li, X. Q. and Jin, Xiangyue and Shen, Xiaojin and Chen, Xiaosha and Sun, Xiaowen and Wang, Xiaoxiang and Song, Xinnan and Zhou, Xinyi and Wang, Xianzu and Shan, Xinxia and Li, Y. K. and Wang, Y. Q. and Wei, Y. X. and Zhang, Yang and Xu, Yanhong and Li, Yao and Zhao, Yao and Sun, Yaofeng and Wang, Yaohui and Yu, Yi and Zhang, Yichao and Shi, Yifan and Xiong, Yiliang and He, Ying and Piao, Yishi and Wang, Yisong and Tan, Yixuan and Ma, Yiyang and Liu, Yiyuan and Guo, Yongqiang and Ou, Yuan and Wang, Yuduan and Gong, Yue and Zou, Yuheng and He, Yujia and Xiong, Yunfan and Luo, Yuxiang and You, Yuxiang and Liu, Yuxuan and Zhou, Yuyang and Zhu, Y. X. and Huang, Yanping and Li, Yaohui and Zheng, Yi and Zhu, Yuchen and Ma, Yunxian and Tang, Ying and Zha, Yukun and Yan, Yuting and Ren, Z. Z. and Ren, Zehui and Sha, Zhangli and Fu, Zhe and Xu, Zhean and Xie, Zhenda and Zhang, Zhengyan and Hao, Zhewen and Ma, Zhicheng and Yan, Zhigang and Wu, Zhiyu and Gu, Zihui and Zhu, Zijia and Liu, Zijun and Li, Zilin and Xie, Ziwei and Song, Ziyang and Pan, Zizheng and Huang, Zhen and Xu, Zhipeng and Zhang, Zhongyu and Zhang, Zhen},
   year={2025},
   month=Sept, pages={633–638} }

@misc{loop1,
      title={Universal Transformers Need Memory: Depth-State Trade-offs in Adaptive Recursive Reasoning}, 
      author={Grigory Sapunov},
      year={2026},
      eprint={2604.21999},
      archivePrefix={arXiv},
      primaryClass={cs.LG},
      url={https://arxiv.org/abs/2604.21999}, 
}

@misc{loop2,
      title={Looped Transformers are Better at Learning Learning Algorithms}, 
      author={Liu Yang and Kangwook Lee and Robert Nowak and Dimitris Papailiopoulos},
      year={2024},
      eprint={2311.12424},
      archivePrefix={arXiv},
      primaryClass={cs.LG},
      url={https://arxiv.org/abs/2311.12424}, 
}

@misc{loop3,
      title={Stability and Generalization in Looped Transformers}, 
      author={Asher Labovich},
      year={2026},
      eprint={2604.15259},
      archivePrefix={arXiv},
      primaryClass={cs.LG},
      url={https://arxiv.org/abs/2604.15259}, 
}

@misc{loop4,
      title={Loop, Think, \& Generalize: Implicit Reasoning in Recurrent-Depth Transformers}, 
      author={Harsh Kohli and Srinivasan Parthasarathy and Huan Sun and Yuekun Yao},
      year={2026},
      eprint={2604.07822},
      archivePrefix={arXiv},
      primaryClass={cs.CL},
      url={https://arxiv.org/abs/2604.07822}, 
}

@misc{loop5,
      title={Reasoning with Latent Thoughts: On the Power of Looped Transformers}, 
      author={Nikunj Saunshi and Nishanth Dikkala and Zhiyuan Li and Sanjiv Kumar and Sashank J. Reddi},
      year={2025},
      eprint={2502.17416},
      archivePrefix={arXiv},
      primaryClass={cs.CL},
      url={https://arxiv.org/abs/2502.17416}, 
}

@misc{loop6,
      title={On Expressive Power of Looped Transformers: Theoretical Analysis and Enhancement via Timestep Encoding}, 
      author={Kevin Xu and Issei Sato},
      year={2025},
      eprint={2410.01405},
      archivePrefix={arXiv},
      primaryClass={cs.LG},
      url={https://arxiv.org/abs/2410.01405}, 
}

@inproceedings{
equilibrium,
title={Equilibrium Language Models},
author={Yikun Jiang and Huanyu Wang and Tianhong Ding and Wenhu Zhang and Yiming Wu and Hanbin Zhao and John C.S. Lui},
booktitle={The Fourteenth International Conference on Learning Representations},
year={2026},
url={https://openreview.net/forum?id=lqJT6xmuH3}
}

@misc{cot,
      title={Chain-of-Thought Prompting Elicits Reasoning in Large Language Models}, 
      author={Jason Wei and Xuezhi Wang and Dale Schuurmans and Maarten Bosma and Brian Ichter and Fei Xia and Ed Chi and Quoc Le and Denny Zhou},
      year={2023},
      eprint={2201.11903},
      archivePrefix={arXiv},
      primaryClass={cs.CL},
      url={https://arxiv.org/abs/2201.11903}, 
}

@misc{loopedtransformer,
      title={Looped Transformers for Length Generalization}, 
      author={Ying Fan and Yilun Du and Kannan Ramchandran and Kangwook Lee},
      year={2025},
      eprint={2409.15647},
      archivePrefix={arXiv},
      primaryClass={cs.LG},
      url={https://arxiv.org/abs/2409.15647}, 
}

@misc{wei2025sim,
      title={SIM-CoT: Supervised Implicit Chain-of-Thought}, 
      author={Xilin Wei and Xiaoran Liu and Yuhang Zang and Xiaoyi Dong and Yuhang Cao and Jiaqi Wang and Xipeng Qiu and Dahua Lin},
      year={2025},
      eprint={2509.20317},
      archivePrefix={arXiv},
      primaryClass={cs.CL},
      url={https://arxiv.org/abs/2509.20317}, 
}

@inproceedings{mohtashami2023cotformer,
  title={Co{T}{F}ormer: {M}ore tokens with attention make up for less depth},
  author={Mohtashami, Amirkeivan and Pagliardini, Matteo and Jaggi, Martin},
  booktitle={Workshop on Advancing Neural Network Training: Computational Efficiency, Scalability, and Resource Optimization (WANT@ NeurIPS 2023)},
  year={2023}
}

@misc{loopedbetter,
      title={Looped Transformers are Better at Learning Learning Algorithms}, 
      author={Liu Yang and Kangwook Lee and Robert Nowak and Dimitris Papailiopoulos},
      year={2024},
      eprint={2311.12424},
      archivePrefix={arXiv},
      primaryClass={cs.LG},
      url={https://arxiv.org/abs/2311.12424}, 
}

@misc{looped1,
      title={Reasoning with Latent Thoughts: On the Power of Looped Transformers}, 
      author={Nikunj Saunshi and Nishanth Dikkala and Zhiyuan Li and Sanjiv Kumar and Sashank J. Reddi},
      year={2025},
      eprint={2502.17416},
      archivePrefix={arXiv},
      primaryClass={cs.CL},
      url={https://arxiv.org/abs/2502.17416}, 
}

@misc{mor,
      title={Mixture-of-Recursions: Learning Dynamic Recursive Depths for Adaptive Token-Level Computation}, 
      author={Sangmin Bae and Yujin Kim and Reza Bayat and Sungnyun Kim and Jiyoun Ha and Tal Schuster and Adam Fisch and Hrayr Harutyunyan and Ziwei Ji and Aaron Courville and Se-Young Yun},
      year={2025},
      eprint={2507.10524},
      archivePrefix={arXiv},
      primaryClass={cs.CL},
      url={https://arxiv.org/abs/2507.10524}, 
}

@misc{nextlat,
      title={Next-Latent Prediction Transformers Learn Compact World Models}, 
      author={Jayden Teoh and Manan Tomar and Kwangjun Ahn and Edward S. Hu and Pratyusha Sharma and Riashat Islam and Alex Lamb and John Langford},
      year={2025},
      eprint={2511.05963},
      archivePrefix={arXiv},
      primaryClass={cs.LG},
      url={https://arxiv.org/abs/2511.05963}, 
}

@misc{chen2025innerthinkingtransformerleveraging,
      title={Inner Thinking Transformer: Leveraging Dynamic Depth Scaling to Foster Adaptive Internal Thinking}, 
      author={Yilong Chen and Junyuan Shang and Zhenyu Zhang and Yanxi Xie and Jiawei Sheng and Tingwen Liu and Shuohuan Wang and Yu Sun and Hua Wu and Haifeng Wang},
      year={2025},
      eprint={2502.13842},
      archivePrefix={arXiv},
      primaryClass={cs.CL},
      url={https://arxiv.org/abs/2502.13842}, 
}

@misc{xu2026formalcomparisonchainthought,
      title={A Formal Comparison Between Chain of Thought and Latent Thought}, 
      author={Kevin Xu and Issei Sato},
      year={2026},
      eprint={2509.25239},
      archivePrefix={arXiv},
      primaryClass={cs.AI},
      url={https://arxiv.org/abs/2509.25239}, 
}

@misc{jeddi2026loopformerelasticdepthloopedtransformers,
      title={LoopFormer: Elastic-Depth Looped Transformers for Latent Reasoning via Shortcut Modulation}, 
      author={Ahmadreza Jeddi and Marco Ciccone and Babak Taati},
      year={2026},
      eprint={2602.11451},
      archivePrefix={arXiv},
      primaryClass={cs.CL},
      url={https://arxiv.org/abs/2602.11451}, 
}

@misc{fu2026thinkathardselectivelatentiterations,
      title={Think-at-Hard: Selective Latent Iterations to Improve Reasoning Language Models}, 
      author={Tianyu Fu and Yichen You and Zekai Chen and Guohao Dai and Huazhong Yang and Yu Wang},
      year={2026},
      eprint={2511.08577},
      archivePrefix={arXiv},
      primaryClass={cs.CL},
      url={https://arxiv.org/abs/2511.08577}, 
}

@article{zhu2025emergence,
  title={Emergence of superposition: {U}nveiling the training dynamics of chain of continuous thought},
  author={Zhu, Hanlin and Hao, Shibo and Hu, Zhiting and Jiao, Jiantao and Russell, Stuart and Tian, Yuandong},
  journal={arXiv preprint arXiv:2509.23365},
  year={2025}
}

@article{zeitoun2026hyperloop,
  title={Hyperloop Transformers},
  author={Zeitoun, Abbas and Torroba-Hennigen, Lucas and Kim, Yoon},
  journal={arXiv preprint arXiv:2604.21254},
  year={2026}
}

@article{knupp2026depth,
  title={{Depth-Recurrent Attention Mixtures: G}iving Latent Reasoning the Attention it Deserves},
  author={Knupp, Jonas and Metzen, Jan Hendrik and Bohn, Jeremias and Groh, Georg and Kersting, Kristian},
  journal={arXiv preprint arXiv:2601.21582},
  year={2026}
}

@article{song2026adaponderlm,
  title={{AdaPonderLM: G}ated Pondering Language Models with Token-Wise Adaptive Depth},
  author={Song, Shixiang and Li, He and Wang, Zitong and Zeng, Boyi and Song, Feichen and Wang, Yixuan and Xu, Zhiqin John and He, Ziwei and Lin, Zhouhan},
  journal={arXiv preprint arXiv:2603.01914},
  year={2026}
}

@inproceedings{zhangmodr,
  title={{MoDr: M}ixture-of-Depth-Recurrent Transformers for Test-Time Reasoning},
  author={Zhang, Xiaojing and Wu, Haifeng and He, Gang and Shen, Jiyang and Lyu, Bochen and Zhu, Zhanxing},
  booktitle={The Fourteenth International Conference on Learning Representations}
}

@inproceedings{zhengchaingpt,
  title={{ChainGPT: D}ual-Reasoning Model with Recurrent Depth and Multi-Rank State Updates},
  author={Zheng, Yunao and Wang, Xiaojie and Ren, Lei and Wei, Chen},
  booktitle={The Fourteenth International Conference on Learning Representations}
}

@misc{geiping2025efficientparallelsamplersrecurrentdepth,
      title={Efficient Parallel Samplers for Recurrent-Depth Models and Their Connection to Diffusion Language Models}, 
      author={Jonas Geiping and Xinyu Yang and Guinan Su},
      year={2025},
      eprint={2510.14961},
      archivePrefix={arXiv},
      primaryClass={cs.LG},
      url={https://arxiv.org/abs/2510.14961}, 
}

@inproceedings{gatmiry2024can,
  title={Can Looped Transformers Learn to Implement Multi-step Gradient Descent for In-context Learning?},
  author={Gatmiry, Khashayar and Saunshi, Nikunj and Reddi, Sashank J and Jegelka, Stefanie and Kumar, Sanjiv},
  booktitle={International Conference on Machine Learning},
  pages={15130--15152},
  year={2024},
  organization={PMLR}
}

@article{zhu2026reasoning,
  title={Reasoning by superposition: {A} theoretical perspective on chain of continuous thought},
  author={Zhu, Hanlin and Hao, Shibo and Hu, Zhiting and Jiao, Jiantao and Russell, Stuart J and Tian, Yuandong},
  journal={Advances in Neural Information Processing Systems},
  volume={38},
  pages={79931--79963},
  year={2026}
}

@misc{hu2025beliefstatetransformer,
      title={The Belief State Transformer}, 
      author={Edward S. Hu and Kwangjun Ahn and Qinghua Liu and Haoran Xu and Manan Tomar and Ada Langford and Jayden Teoh and Bryon Xu and David Yan and Dinesh Jayaraman and Alex Lamb and John Langford},
      year={2025},
      eprint={2410.23506},
      archivePrefix={arXiv},
      primaryClass={cs.LG},
      url={https://arxiv.org/abs/2410.23506}, 
}

@misc{fleuret2025freetransformer,
      title={The Free Transformer}, 
      author={François Fleuret},
      year={2025},
      eprint={2510.17558},
      archivePrefix={arXiv},
      primaryClass={cs.LG},
      url={https://arxiv.org/abs/2510.17558}, 
}

@book{krantz2002implicit,
  title={The implicit function theorem: history, theory, and applications},
  author={Krantz, Steven George and Parks, Harold R},
  volume={202},
  number={11},
  year={2002},
  publisher={Springer}
}

@misc{kaiser2016neuralgpuslearnalgorithms,
      title={Neural GPUs Learn Algorithms}, 
      author={Łukasz Kaiser and Ilya Sutskever},
      year={2016},
      eprint={1511.08228},
      archivePrefix={arXiv},
      primaryClass={cs.LG},
      url={https://arxiv.org/abs/1511.08228}, 
}

@misc{lan2020albertlitebertselfsupervised,
      title={ALBERT: A Lite BERT for Self-supervised Learning of Language Representations}, 
      author={Zhenzhong Lan and Mingda Chen and Sebastian Goodman and Kevin Gimpel and Piyush Sharma and Radu Soricut},
      year={2020},
      eprint={1909.11942},
      archivePrefix={arXiv},
      primaryClass={cs.CL},
      url={https://arxiv.org/abs/1909.11942}, 
}

@misc{dabre2018recurrentstackinglayerscompact,
      title={Recurrent Stacking of Layers for Compact Neural Machine Translation Models}, 
      author={Raj Dabre and Atsushi Fujita},
      year={2018},
      eprint={1807.05353},
      archivePrefix={arXiv},
      primaryClass={cs.CL},
      url={https://arxiv.org/abs/1807.05353}, 
}

@misc{takase2023lessonsparametersharinglayers,
      title={Lessons on Parameter Sharing across Layers in Transformers}, 
      author={Sho Takase and Shun Kiyono},
      year={2023},
      eprint={2104.06022},
      archivePrefix={arXiv},
      primaryClass={cs.CL},
      url={https://arxiv.org/abs/2104.06022}, 
}

@misc{bae2025relaxedrecursivetransformerseffective,
      title={Relaxed Recursive Transformers: Effective Parameter Sharing with Layer-wise LoRA}, 
      author={Sangmin Bae and Adam Fisch and Hrayr Harutyunyan and Ziwei Ji and Seungyeon Kim and Tal Schuster},
      year={2025},
      eprint={2410.20672},
      archivePrefix={arXiv},
      primaryClass={cs.CL},
      url={https://arxiv.org/abs/2410.20672}, 
}

@inproceedings{
zheng2026chaingpt,
title={Chain{GPT}: Dual-Reasoning Model with Recurrent Depth and Multi-Rank State Updates},
author={Yunao Zheng and Xiaojie Wang and Lei Ren and Chen Wei},
booktitle={The Fourteenth International Conference on Learning Representations},
year={2026},
url={https://openreview.net/forum?id=kdZbxizwGK}
}

@inproceedings{
zhang2026modr,
title={MoDr: Mixture-of-Depth-Recurrent Transformers for Test-Time Reasoning},
author={Xiaojing Zhang and Haifeng Wu and Gang He and Jiyang Shen and Bochen Lyu and Zhanxing Zhu},
booktitle={The Fourteenth International Conference on Learning Representations},
year={2026},
url={https://openreview.net/forum?id=9Pba4rcQbE}
}

@misc{song2026adaponderlmgatedponderinglanguage,
      title={AdaPonderLM: Gated Pondering Language Models with Token-Wise Adaptive Depth}, 
      author={Shixiang Song and He Li and Zitong Wang and Boyi Zeng and Feichen Song and Yixuan Wang and Zhiqin John Xu and Ziwei He and Zhouhan Lin},
      year={2026},
      eprint={2603.01914},
      archivePrefix={arXiv},
      primaryClass={cs.CL},
      url={https://arxiv.org/abs/2603.01914}, 
}

@misc{moosa2026understandingdynamiccomputeallocation,
      title={Understanding Dynamic Compute Allocation in Recurrent Transformers}, 
      author={Ibraheem Muhammad Moosa and Suhas Lohit and Ye Wang and Moitreya Chatterjee and Wenpeng Yin},
      year={2026},
      eprint={2602.08864},
      archivePrefix={arXiv},
      primaryClass={cs.CL},
      url={https://arxiv.org/abs/2602.08864}, 
}

@misc{knupp2026depthrecurrentattentionmixturesgiving,
      title={Depth-Recurrent Attention Mixtures: Giving Latent Reasoning the Attention it Deserves}, 
      author={Jonas Knupp and Jan Hendrik Metzen and Jeremias Bohn and Georg Groh and Kristian Kersting},
      year={2026},
      eprint={2601.21582},
      archivePrefix={arXiv},
      primaryClass={cs.AI},
      url={https://arxiv.org/abs/2601.21582}, 
}

@misc{wu2025parallellooptransformerefficient,
      title={Parallel Loop Transformer for Efficient Test-Time Computation Scaling}, 
      author={Bohong Wu and Mengzhao Chen and Xiang Luo and Shen Yan and Qifan Yu and Fan Xia and Tianqi Zhang and Hongrui Zhan and Zheng Zhong and Xun Zhou and Siyuan Qiao and Xingyan Bin},
      year={2025},
      eprint={2510.24824},
      archivePrefix={arXiv},
      primaryClass={cs.CL},
      url={https://arxiv.org/abs/2510.24824}, 
}

@misc{zeitoun2026hyperlooptransformers,
      title={Hyperloop Transformers}, 
      author={Abbas Zeitoun and Lucas Torroba-Hennigen and Yoon Kim},
      year={2026},
      eprint={2604.21254},
      archivePrefix={arXiv},
      primaryClass={cs.LG},
      url={https://arxiv.org/abs/2604.21254}, 
}

@misc{giannou2023loopedtransformersprogrammablecomputers,
      title={Looped Transformers as Programmable Computers}, 
      author={Angeliki Giannou and Shashank Rajput and Jy-yong Sohn and Kangwook Lee and Jason D. Lee and Dimitris Papailiopoulos},
      year={2023},
      eprint={2301.13196},
      archivePrefix={arXiv},
      primaryClass={cs.LG},
      url={https://arxiv.org/abs/2301.13196}, 
}

@misc{yang2024loopedtransformersbetterlearning,
      title={Looped Transformers are Better at Learning Learning Algorithms}, 
      author={Liu Yang and Kangwook Lee and Robert Nowak and Dimitris Papailiopoulos},
      year={2024},
      eprint={2311.12424},
      archivePrefix={arXiv},
      primaryClass={cs.LG},
      url={https://arxiv.org/abs/2311.12424}, 
}

@misc{gatmiry2024roledepthloopingincontext,
      title={On the Role of Depth and Looping for In-Context Learning with Task Diversity}, 
      author={Khashayar Gatmiry and Nikunj Saunshi and Sashank J. Reddi and Stefanie Jegelka and Sanjiv Kumar},
      year={2024},
      eprint={2410.21698},
      archivePrefix={arXiv},
      primaryClass={cs.LG},
      url={https://arxiv.org/abs/2410.21698}, 
}

@misc{huang2025transformerslearnimplementmultistep,
      title={Transformers Learn to Implement Multi-step Gradient Descent with Chain of Thought}, 
      author={Jianhao Huang and Zixuan Wang and Jason D. Lee},
      year={2025},
      eprint={2502.21212},
      archivePrefix={arXiv},
      primaryClass={cs.LG},
      url={https://arxiv.org/abs/2502.21212}, 
}

@misc{xu2025expressivepowerloopedtransformers,
      title={On Expressive Power of Looped Transformers: Theoretical Analysis and Enhancement via Timestep Encoding}, 
      author={Kevin Xu and Issei Sato},
      year={2025},
      eprint={2410.01405},
      archivePrefix={arXiv},
      primaryClass={cs.LG},
      url={https://arxiv.org/abs/2410.01405}, 
}

@misc{merrill2025littledepthgoeslong,
      title={A Little Depth Goes a Long Way: The Expressive Power of Log-Depth Transformers}, 
      author={William Merrill and Ashish Sabharwal},
      year={2025},
      eprint={2503.03961},
      archivePrefix={arXiv},
      primaryClass={cs.LG},
      url={https://arxiv.org/abs/2503.03961}, 
}

@misc{sapunov2026universaltransformersneedmemory,
      title={Universal Transformers Need Memory: Depth-State Trade-offs in Adaptive Recursive Reasoning}, 
      author={Grigory Sapunov},
      year={2026},
      eprint={2604.21999},
      archivePrefix={arXiv},
      primaryClass={cs.LG},
      url={https://arxiv.org/abs/2604.21999}, 
}

@misc{labovich2026stabilitygeneralizationloopedtransformers,
      title={Stability and Generalization in Looped Transformers}, 
      author={Asher Labovich},
      year={2026},
      eprint={2604.15259},
      archivePrefix={arXiv},
      primaryClass={cs.LG},
      url={https://arxiv.org/abs/2604.15259}, 
}

@misc{pappone2025twoscalelatentdynamicsrecurrentdepth,
      title={Two-Scale Latent Dynamics for Recurrent-Depth Transformers}, 
      author={Francesco Pappone and Donato Crisostomi and Emanuele Rodolà},
      year={2025},
      eprint={2509.23314},
      archivePrefix={arXiv},
      primaryClass={cs.LG},
      url={https://arxiv.org/abs/2509.23314}, 
}

@misc{han2026hierarchicalvsflatiteration,
      title={Hierarchical vs. Flat Iteration in Shared-Weight Transformers}, 
      author={Sang-Il Han},
      year={2026},
      eprint={2604.14442},
      archivePrefix={arXiv},
      primaryClass={cs.CL},
      url={https://arxiv.org/abs/2604.14442}, 
}

@misc{goyal2026eltelasticloopedtransformers,
      title={ELT: Elastic Looped Transformers for Visual Generation}, 
      author={Sahil Goyal and Swayam Agrawal and Gautham Govind Anil and Prateek Jain and Sujoy Paul and Aditya Kusupati},
      year={2026},
      eprint={2604.09168},
      archivePrefix={arXiv},
      primaryClass={cs.CV},
      url={https://arxiv.org/abs/2604.09168}, 
}

@misc{tang2026looprptreinforcementpretraininglooped,
      title={LoopRPT: Reinforcement Pre-Training for Looped Language Models}, 
      author={Guo Tang and Shixin Jiang and Heng Chang and Nuo Chen and Yuhan Li and Huiming Fan and Jia Li and Ming Liu and Bing Qin},
      year={2026},
      eprint={2603.19714},
      archivePrefix={arXiv},
      primaryClass={cs.CL},
      url={https://arxiv.org/abs/2603.19714}, 
}

@misc{rauba2026tinyautoregressiverecursivemodels,
      title={Tiny Autoregressive Recursive Models}, 
      author={Paulius Rauba and Claudio Fanconi and Mihaela van der Schaar},
      year={2026},
      eprint={2603.08082},
      archivePrefix={arXiv},
      primaryClass={cs.LG},
      url={https://arxiv.org/abs/2603.08082}, 
}

@misc{zhu2025reasoningsuperpositiontheoreticalperspective,
      title={Reasoning by Superposition: A Theoretical Perspective on Chain of Continuous Thought}, 
      author={Hanlin Zhu and Shibo Hao and Zhiting Hu and Jiantao Jiao and Stuart Russell and Yuandong Tian},
      year={2025},
      eprint={2505.12514},
      archivePrefix={arXiv},
      primaryClass={cs.LG},
      url={https://arxiv.org/abs/2505.12514}, 
}

@misc{zeng2026ponderlmpretraininglanguagemodels,
      title={PonderLM: Pretraining Language Models to Ponder in Continuous Space}, 
      author={Boyi Zeng and Shixiang Song and Siyuan Huang and Yixuan Wang and He Li and Ziwei He and Xinbing Wang and Zhiyu Li and Zhouhan Lin},
      year={2026},
      eprint={2505.20674},
      archivePrefix={arXiv},
      primaryClass={cs.CL},
      url={https://arxiv.org/abs/2505.20674}, 
}

@misc{ghugare2026roleiterativecomputationreinforcement,
      title={On the Role of Iterative Computation in Reinforcement Learning}, 
      author={Raj Ghugare and Michał Bortkiewicz and Alicja Ziarko and Benjamin Eysenbach},
      year={2026},
      eprint={2602.05999},
      archivePrefix={arXiv},
      primaryClass={cs.LG},
      url={https://arxiv.org/abs/2602.05999}, 
}

@misc{tur2026recurrentdepthvlaimplicittesttime,
      title={Recurrent-Depth VLA: Implicit Test-Time Compute Scaling of Vision-Language-Action Models via Latent Iterative Reasoning}, 
      author={Yalcin Tur and Jalal Naghiyev and Haoquan Fang and Wei-Chuan Tsai and Jiafei Duan and Dieter Fox and Ranjay Krishna},
      year={2026},
      eprint={2602.07845},
      archivePrefix={arXiv},
      primaryClass={cs.RO},
      url={https://arxiv.org/abs/2602.07845}, 
}

@misc{komisarczyk2026recursiveinferencemachinesneural,
      title={Recursive Inference Machines for Neural Reasoning}, 
      author={Mieszko Komisarczyk and Saurabh Mathur and Maurice Kraus and Sriraam Natarajan and Kristian Kersting},
      year={2026},
      eprint={2603.05234},
      archivePrefix={arXiv},
      primaryClass={cs.LG},
      url={https://arxiv.org/abs/2603.05234}, 
}

@misc{cameron2026stepforwardksteps,
      title={One Step Forward and K Steps Back: Better Reasoning with Denoising Recursion Models}, 
      author={Chris Cameron and Wangzheng Wang and Nikita Ivanov and Ashmita Bhattacharyya and Didier Chételat and Yingxue Zhang},
      year={2026},
      eprint={2604.18839},
      archivePrefix={arXiv},
      primaryClass={cs.LG},
      url={https://arxiv.org/abs/2604.18839}, 
}

@misc{williams2026prioritizeprocessjustoutcome,
      title={Prioritize the Process, Not Just the Outcome: Rewarding Latent Thought Trajectories Improves Reasoning in Looped Language Models}, 
      author={Jonathan Williams and Esin Tureci},
      year={2026},
      eprint={2602.10520},
      archivePrefix={arXiv},
      primaryClass={cs.LG},
      url={https://arxiv.org/abs/2602.10520}, 
}
